\documentclass[a4paper,12pt]{article}

\usepackage{lipsum}
\usepackage{amsfonts}
\usepackage{graphicx}
\usepackage{epstopdf}
\usepackage{algorithmic}
\usepackage{relsize}
\usepackage{tikz}
\usepackage{tzplot}
\usepackage{setspace}

\usepackage{tabularx}
\usepackage{booktabs}

\usetikzlibrary{positioning}
\ifpdf
  \DeclareGraphicsExtensions{.eps,.pdf,.png,.jpg}
\else
  \DeclareGraphicsExtensions{.eps}
\fi

\usepackage{enumitem}
\usepackage{tikz}
\usepackage{tkz-euclide}
\setlist[itemize]{leftmargin=.5in}

\usepackage{amsmath,amsthm}
\usepackage{amssymb}
\usepackage{hyperref}
\usepackage{algorithm}
\usepackage{geometry}
\usepackage{esint}
\usepackage{array}
\usepackage[linewidth=2pt]{mdframed}

\newmdenv[
topline=false,
bottomline=false,
rightline=false,
skipabove=\topsep,
skipbelow=\topsep,
linewidth=4
]{siderules}
\usepackage{hyperref}
\usepackage{xcolor}
\usepackage[most]{tcolorbox}
\usepackage[nottoc]{tocbibind}

\newcommand{\supp}{\text{supp}}

\newtheorem*{theorem*}{Theorem}
\newtheorem{theorem}{Theorem}
\newtheorem{proposition}[theorem]{Proposition}
\newtheorem{remark}[theorem]{Remark}
\newtheorem{definition}[theorem]{Definition}
\newtheorem{lemma}[theorem]{Lemma}
\newtheorem{corollary}[theorem]{Corollary}

\newcommand*\dx{\mathop{}\!\mathrm{d}}

\DeclareMathOperator{\relu}{ReLU}

\title{On the growth of the parameters of a class\\of approximating ReLU neural networks}
\author{Martin Holler \thanks{IDea\_Lab - The Interdisciplinary Digital Lab at the University of Graz, University of Graz, Austria.
\{\href{mailto:martin.holler@uni-graz.at}{martin.holler@uni-graz.at}, \href{mailto:erion.morina@uni-graz.at}{erion.morina@uni-graz.at}\}
} \and Erion Morina \footnotemark[1]}

\usepackage{amsopn}

\ifpdf
\hypersetup{
	pdftitle={On the growth of the parameters of a class of approximating ReLU neural networks},
	pdfauthor={M. Holler and E. Morina}
}
\fi
\begin{document}	
	\maketitle
	\begin{abstract}
		This work focuses on the analysis of a class of fully connected feed forward ReLU neural networks as they approximate a given, smooth function. In contrast to conventionally studied universal approximation properties under increasing architectures, e.g., in terms of width or depth of the networks, we are concerned with the asymptotic growth of the norm of the parameters of approximating networks. Such results are of interest, e.g., for error analysis or consistency results for neural network training. The main result of our work is that, for a class of ReLU architectures with state of the art approximation error, the norm of the realizing parameters grows at most polynomially. The obtained rate with respect to a normalized network size is compared to existing results and is shown to be superior in most cases, in particular for high dimensional input.
	\end{abstract}
	\section{Introduction}
	It is well known that certain neural network architectures have a universal approximation property, i.e., functions of certain regularity may be approximated arbitrarily well with respect to appropriate norms. This can be achieved by increasing the complexity of the underlying neural networks, where complexity is usually described by the network size in terms of width, depth, or number of weights and neurons when it comes to fully connected feed forward neural networks.
	In the classical work \cite{mhaskar}, shallow approximations of Sobolev functions with respect to $L^p$-norms are studied. See also \cite{guehring20} for deep ReLU approximation of Sobolev functions with respect to general Sobolev norms and the references therein. For the (nearly) optimal approximation of piecewise smooth functions by ReLU networks see \cite{lu_main, Petersen18}. The works in \cite{devore21,elbraecther21, gribonval20} together with their references give a comprehensive overview of the approximation theory based on neural networks.\\
	 A question that has not been given much attention in the current literature is how the norm of the parameters realizing the approximating networks behaves asymptotically. We refer to Section \ref{sec:compare} below for an overview of related work.
	 This question, however, is highly relevant, e.g. in view of consistency results or a full error analysis for neural network training. As a simple example for the former, consider a network $f_{S,\theta_S}$, with $S$ denoting its size, e.g width times depth, and $\theta_S$ its (size-dependent) parameters, that is trained using regularization of the parameters, called \emph{weight decay} in machine learning literature. Assuming that $\delta>0$ is a parameter describing the data quality, e.g. the noise level and/or the number of data points, let the training problem be given as
	 \[ \min _{\theta_S} \mathcal{L}^\delta (f_{S,\theta_S}) + \lambda(\delta) \|\theta_S\|_2^2.
	 \] 
	 Here, $\mathcal{L}^\delta$ and $\lambda(\delta)$ denote the loss and regularization parameter, respectively, both depending on the data quality. In order to approximate an underlying ground-truth function $f \approx f_{S,\theta_S}$ as the data quality increases, i.e., as $\delta \rightarrow 0$, one usually considers $S \rightarrow \infty$ together with approximation capabilities of the network. In view of the regularization, however, also a suitable parameter choice $\lambda(\delta)$ such that $\lambda(\delta) \rightarrow 0$ as $\delta \rightarrow 0$ needs to be made. This choice requires an upper bound on the norm of the parameters $\|\theta_S\|_2^2$, as is provided e.g. by our work.
	 
A second, more precise example where an upper bound on the norm of the parameters is necessary to obtain generalization bounds are the results of \cite{jentzen_riekert, jentzen_welti}. Here, a full error analysis of deep learning for empirical risk minimization is provided, with the underlying networks being trained based on Stochastic Gradient Descent (SGD) with random initializations and with the results holding in the probabilistic sense.
A simplification of the result in \cite[Theorem 1.1]{jentzen_riekert} is given as follows. For some $d\in\mathbb{N}$ let $f:[0,1]^d\to [0,1]$ be Lipschitz continuous. Assume that $(f_{N,L,\theta_{N,L}})_{N,L\in \mathbb{N}}$ is a sequence of networks such that $f_{N,L,\theta_{N,L}}$ is a network of width $N$ and depth $L$ with parameters given as $\theta_{N,L}$. The parameters $\theta_{N,L}$ are trained via SGD,
minimizing the empirical risk over $M$ given i.i.d. training samples, under the constraint that $\Vert\theta_{N,L}\Vert_\infty\leq c_{N,L}$. Here $c_{N,L}$ is some a priori chosen bound on the supremum norm of the parameters which is larger than the Lipschitz constant of $f$. Let further $\mathbf{A}(N,L)c_{N,L}$ denote some upper bound on the approximation error of $f$ in terms of $N$ and $L$, approaching zero for increasing $N$ and $L$. Note that $\mathbf{A}$ depends on the approximating architecture and in general stronger regularity assumptions of $f$ than Lipschitz-continuity are required. Furthermore, let $\mathbf{O}(N,L,K)c_{N,L}^{L+1}$ be some upper bound on the optimization error, where $K$ is the number of random initializations of SGD, and let $\mathbf{G}(N,L,M)c_{N,L}$ be some upper bound on the generalization error. Explicit characterizations of $\mathbf{A}$, $\mathbf{O}$ and $\mathbf{G}$ in terms of $N,L,M,K$, respectively, are given in \cite[Theorem 1.1]{jentzen_riekert}. The error functions $\mathbf{O}$ and $\mathbf{G}$ approach zero for fixed $N,L\in \mathbb{N}$ as $K$ and $M$ increase. Under these assumptions it holds true by \cite[Theorem 1.1]{jentzen_riekert} that
	 \begin{align*}
	 	\mathbb{E}(\Vert f_{N,L,\theta_{N,L}}-f\Vert_{L^1([0,1]^d, \mathbb{P})})\leq \mathbf{A}(N,L)c_{N,L}+\mathbf{O}(N,L,K)c_{N,L}^{L+1}+\mathbf{G}(N,L,M)c_{N,L}.
	 \end{align*}
	 This shows that in order to bound the left hand side, one does not only require estimates on $\mathbf{A}(N,L)$, $\mathbf{O}(N,L,K)$, and $\mathbf{G}(N,L,M)$, but also on $c_{N,L}$ bounding the supremum norm of the parameters $\theta_{N,L}$ of the approximating neural network.

Beyond these two examples, a control on the norm of the parameters makes the size of a network a meaningful measure of its complexity. A network is determined by its parameters only up to the precision with which these are stored, and storing a parameter of magnitude $c$ at accuracy $\epsilon$ requires $\mathcal{O}(\log(c/\epsilon))$ bits. Bounding only the number of weights is not sufficient to characterize the amount of information a network carries. From an information-theoretic perspective the weights are required to be \emph{encodable}, i.e., representable by a number of bits that grows at most logarithmically in the network size or the inverse approximation accuracy, see \cite{boelcskei19optimal, elbraecther21, raslan21}. In \cite{elbraecther21}, for instance, networks are considered whose depth grows at most polylogarithmically and the magnitude of their parameters at most polynomially in their connectivity, i.e., in the number of nonzero weights \cite[Definition VI.2]{elbraecther21}. The approximation rate achievable under these restrictions does not exceed the optimal exponent of the approximated function class (see \cite[Definition IV.1]{elbraecther21}) in the sense of \emph{Kolmogorov-Donoho} rate-distortion theory \cite[Theorem
VI.4]{elbraecther21}. In a related vein, \cite[Theorem 3.5]{raslan21} transfers entropy lower bounds to lower bounds on the number of nonzero encodable weights necessary for a prescribed accuracy, while the corresponding upper bound \cite[Proposition 4.8]{raslan21} is obtained under parameters bounded polynomially in the inverse accuracy. Statements of this type are not meaningful without a restriction on the magnitude of the parameters, since otherwise arbitrarily much information may be hidden in a single parameter.

		It is clear that bounds on the norm of parameters of approximating neural networks must always be analyzed jointly with the network depth and width. Indeed, networks with large parameter bounds may be expanded in width and depth such that they describe the same function, but such that the modified parameters are considerably smaller. Conversely, and less obviously, unbounded parameters may compensate for a fixed architecture to an extent that makes the network size meaningless as a measure of complexity. This is showcased in \cite{yarotsky21elementary} for \emph{superexpressive} activation functions. A finite family of activation functions is called superexpressive for dimension $d$ if there is a single fixed architecture with $d$ input neurons such that every $f\in \mathcal{C}([0,1]^d)$ can be approximated in the supremum norm to any accuracy by only adjusting the weights, without enlarging the network \cite[Definition 1]{yarotsky21elementary}. An example of a superexpressive family is $\{\sin,\arcsin\}$ as shown in \cite[Theorem 3]{yarotsky21elementary}. An earlier, non-elementary example is provided in \cite{maiorov99}.

	As pointed out in \cite{yarotsky21elementary}, such results are of a purely theoretical nature, since the entire approximation complexity is hidden in a very particular choice of the weights which is numerically unfeasible in general. A related observation is that the set of functions realized by networks of a fixed architecture is, for most activation functions used in practice, not closed, so that best approximations need not exist and approximating sequences may only be realizable by parameters diverging in norm, see \cite{petersen21}. For the ReLU activation function considered here superexpressivity does not occur, since networks of fixed width and depth realize piecewise affine functions with an a priori bounded number of pieces. More generally, no family of finitely many piecewise \emph{Pfaffian} activations is superexpressive \cite[Theorem 5]{yarotsky21elementary}.
		
		Altogether, neither the size of a network nor the magnitude of its parameters is on its own an adequate measure of complexity. In order to account for this, our strategy is to consider networks that achieve optimal or nearly optimal approximation results with respect to width and depth, and analyze the asymptotic behavior of parameters realizing those networks.
	Specifically, we analyze the asymptotic behavior of the parameters realizing the approximating network architectures studied in \cite{lu_main} and \cite{mhaskar}. An advantage with these works is that they are based on explicit constructions which allow for successive estimation of the occurring parameters. The approximation scheme in \cite{mhaskar} provides an optimal approximation error for single hidden layer neural networks whereas \cite{lu_main} establishes the optimal order of approximation in terms of width and depth, up to logarithmic factors. The latter is achieved by a bit-extraction technique in combination with a localization on a partition of $[0,1]^d$ into cubes, given by the piecewise-constant projection of each cell onto one of its corners. Since a ReLU FNN realizes a continuous function, it cannot reproduce this projection exactly, but only outside a thin neighborhood of the cell boundaries, the \emph{trifling region} (see equation \eqref{trifling_region}), across which it transitions with a correspondingly large slope. The extension of the approximation to all of $[0,1]^d$ is achieved in \cite[Theorem 2.1]{lu_main}, and the corresponding estimates on the parameters constitute a substantial part of our analysis. An alternative localization is provided by the partitions of unity used in \cite{guehring20, raslan21}, whose members are continuous and piecewise affine, so that no trifling region arises, but which yield different approximation rates.

	Note that we do not make any statement on optimality of the asymptotical behavior of the realizing parameters. The goal is to get an insight in the asymptotic behavior of the norm of the parameters realizing the approximating schemes in \cite{guehring20,lu_main,mhaskar}, compared to the current state in literature (and not to cover analysis on e.g. \textit{all} deep ReLU networks).

	For the shallow approximation scheme with smooth activation functions explicitly introduced in \cite{mhaskar}, we show exemplarily in Appendix \ref{subsec:shallow} the negative result that, for the Gaussian and logistic activation function, the parameters realizing the neural networks defined in \cite{mhaskar} approximating a specifically chosen function grow at least exponentially under mild assumptions. Note that we choose the term \textit{negative result} on purpose and not \textit{counterexample}. Our negative result shows that, for a classical choice of approximation, a polynomial growth of the norms of the parameters
cannot be taken for granted. It is, however, not a counterexample, which would need to exclude existence of a different parameter
choice for the same architecture with improved rates.

	The deep approximation scheme of \cite{guehring20}, where the local approximations are stitched together by a partition of unity instead of by a trifling region, comes without a bound on the norm of the realizing parameters. We derive one in Appendix \ref{sec:pou}, obtaining a slower growth than for \cite{lu_main}, at the price of an approximation scheme which does not attain the optimal order of approximation in width and depth, and whose architecture is in addition determined by the accuracy alone, so that no trade-off between width and depth is available.

	For the deep approximation introduced in \cite{lu_main} based on the ReLU activation function we provide a small modification with slightly increased depth, such that the realizing parameters of the modified approximating networks grow at most polynomially. A simplified version of our main result in this context, which is based on \cite[Theorem 1.1]{lu_main}, is given as follows.
	\begin{theorem*}[Simplification of Theorem \ref{theorem_lu_param_growth}]
		Let $d,q\in \mathbb{N}$ and $f\in \mathcal{C}^q([0,1]^d)$. Then for any $N,L\in \mathbb{N}$ there exists a ReLU feed forward neural network $f_{N,L}$ with width of order $N\log N$ and depth of order $L^2\log L$ such that
		\[
			\Vert f-f_{N,L}\Vert_{L^\infty([0,1]^d)}=\mathcal{O}(N^{-2q/d}L^{-2q/d}) \] 
			\[\text{and}\quad \|\theta_{f_{N,L}}\|_\infty = \mathcal{O}(\max(N^{(6q-3)/d}L^{(6q-2)/d}, NL(N+L^2))),
		\]
		where $\theta_{f_{N,L}}$ denotes a suitable choice of parameters $\theta$ realizing the network $f_{N,L}$.
		The approximation error is, up to logarithmic factors in the width and the depth, of optimal order, and it is attained for any prescribed $N$ and $L$.
	\end{theorem*}
	\paragraph{Scope of the paper.} In Section \ref{sec:compare} we compare the result above to existing results in the literature. In Section \ref{sec:param_growth} we provide our main results, in particular the analysis on the asymptotical behavior of the parameters of the approximation with deep networks introduced in \cite{lu_main}.
	Appendix \ref{subsec:shallow} provides the negative result on the single hidden layer networks of \cite{mhaskar}. Appendix \ref{sec:pou} provides the analysis of the partition of unity based approximation of \cite{guehring20}.
	\section{Comparison to existing literature}
	\label{sec:compare}
This section provides works that also deal with the asymptotic behavior of the norm of parameters of approximating networks and compares their result with ours. 
The work \cite{deryck21} considers the approximation of certain Sobolev-regular functions by shallow feed forward tanh-type neural networks. In particular, the approximation result in \cite[Theorem 5.1]{deryck21} is provided under parameters whose norm grows at most polynomially in terms of the width. The work \cite{belomestny23} deals with the approximation of Hölder-smooth functions by feed forward neural networks with piecewise polynomial activation functions among others. In \cite[Theorem 2]{belomestny23} the approximation result is achieved with uniformly bounded weights. 
For so-called $(p,C)-$smooth functions (see e.g. \cite[Definition 1]{langer21}) an approximation result with sigmoidal activation functions is provided in \cite[Theorem 1]{langer21} for at most polynomially growing parameters in terms of the width. In \cite[Proposition 4.8]{raslan21} an approximation result for Sobolev-regular functions under activation functions enabling the construction of exact/exponential/polynomial partitions of unity (see \cite[Definition 4.1]{raslan21}) is established with parameters whose norm is bounded polynomially in terms of the number of non-zero weights. The approximation result of \cite[Theorem 4.1]{guehring20} for Sobolev-regular functions, which is based on a localization by a partition of unity as well, does not contain a statement on the norm of the realizing parameters. We derive it in Appendix \ref{sec:pou}, see Proposition \ref{prop:pou_param_growth}, and include it in the comparison below.\\
	
For comparing the above approximation result to ours, it is important that the functions that are approximated attain the same regularity, that the norm which measures the approximation error is the same or at least comparable, and that the hyperparameter of the approximating architecture with respect to which the approximation error decreases is transformed to the same order of complexity.

Consequently, we compare the approximation results for functions $f\in \mathcal{C}^q([0,1]^d)$, since they meet the 
regularity requirements of \cite{lu_main} (smoothness), \cite{deryck21,raslan21} (Sobolev-regularity) and \cite{belomestny23, langer21} (Hölder-regularity since the derivative of order $q-1$ is Lipschitz-continuous).
A comparable norm for the approximation error in \cite{belomestny23,deryck21,raslan21, langer21,lu_main} is the supremum norm. A summary of the approximation results of the different works and of Theorem \ref{theorem_lu_param_growth} is provided in Table \ref{comparison:1} and Table \ref{comparison:2}.
	\begin{table}[h!]
		\small
		\centering
		\setlength{\tabcolsep}{3pt}
		\begin{tabular}{c c c c c c}
			\toprule
			Result & Width & Depth & Approximation & Growth of parameters & Activation\\
			\midrule
			Th. \ref{theorem_lu_param_growth} &$\mathcal{O}(N)$ & $\mathcal{O}(L)$ &$\mathcal{O}(N^{\frac{-2q}{d(1+\delta)}}L^{\frac{-q}{d(1+\delta)}})$ & $\mathcal{O}(N^{\frac{6q-3}{d}}L^{\frac{3q-1}{d}}\lor NL^{\frac{1}{2}}(N\lor L))$ & $\relu$\\
			\cite{belomestny23}& $\mathcal{O}(N)$ & $\mathcal{O}(1)$ & $\mathcal{O}(N^{-q/d})$ &$\mathcal{O}(1)$ &\text{ReQU}\\
			\cite{deryck21}& $\mathcal{O}(N)$ & $3$ & $\mathcal{O}(N^{-q/d})$ & $\mathcal{O}(N^{(d+q^2)/2})$ &$\tanh$\\
			\cite{langer21}& $\mathcal{O}(N)$ & $\mathcal{O}(1)$ & $\mathcal{O}(N^{-2q/d})$ & $\mathcal{O}(N^{(16q+2d+9)/d})$& $\frac{1}{1+\exp(-x)}$\\
			\bottomrule
		\end{tabular}
		\caption{Comparison of state of the art results on growth of parameters realizing approximations to $f\in\mathcal{C}^q([0,1]^d)$ with normalized width.}
		\label{comparison:1}
	\end{table}
		\begin{table}[h!]
			\small
		\centering
	\begin{tabular}{c c c c c}
		\toprule
		Result & Nonzero weights & Approximation& Growth of parameters & Activation\\
		\midrule
		Th. \ref{theorem_lu_param_growth} &$\mathcal{O}(W)$&$\mathcal{O}(W^{-q/d})$ & $\mathcal{O}(W^{\frac{9q-4}{3d}\lor \frac{5}{6}})$ & ReLU\\
		\cite{raslan21} &$\mathcal{O}(W)$&$\mathcal{O}(W^{-q/d})$ & $\mathcal{O}(W^{3+2q/d})$ & RePU, soft+\\
		\cite{guehring20} &$\mathcal{O}(W)$&$\mathcal{O}((W/\log W)^{-q/d})$ & $\mathcal{O}((W/\log W)^{1/d})$ & ReLU\\
		\bottomrule
	\end{tabular}
	\caption{Comparison of state of the art result on growth of parameters realizing approximations to $f\in\mathcal{C}^q([0,1]^d)$ with normalized number of nonzero weights.}
	\label{comparison:2}
	\end{table}
	Note that the asymptotical bound for the approximation error of Theorem \ref{theorem_lu_param_growth} in Table \ref{comparison:1} holds for any $\delta\in (0,1)$ (as the logarithm of $N$ grows slower than any positive power of $N$) and that the result in \cite{belomestny23} is valid for $q\geq 3$.
Also note that, in Theorem \ref{theorem_lu_param_growth},  it is possible to vary the depth of the approximating architecture as opposed to the results in \cite{belomestny23,deryck21,langer21}, but for the sake of comparison we consider a constant depth.

The approximation error in Table \ref{comparison:1} relative to the width of Theorem \ref{theorem_lu_param_growth} is better than in \cite{belomestny23, deryck21} but slightly worse than in \cite{langer21}.
Regarding the growth of parameters, in case $6q<2d+3$, one can observe that the bound for Theorem \ref{theorem_lu_param_growth} grows slower than of \cite{deryck21} except for the case $q=1, d=2$. In case $6q\geq 2d+3$, the bound for Theorem \ref{theorem_lu_param_growth} grows slower than of \cite{deryck21} except for $d=2$ with $q\in \{2,3,4\}$ and $d=1$ with $q\leq 11$. Regarding \cite{langer21}, one can observe that the growth of parameters in Theorem \ref{theorem_lu_param_growth} is always slower than the one of \cite{langer21}. Regarding a comparison to \cite{belomestny23}, it is interesting to see that \cite{belomestny23} even yields uniformly bounded parameters in $[-1,1]$ (though with slightly worse approximation error compared to Theorem \ref{theorem_lu_param_growth}). In the following, we detail the main differences between the result of \cite{belomestny23} and Theorem \ref{theorem_lu_param_growth}.
	
One difference is that Theorem \ref{theorem_lu_param_growth} allows for adjusting the depth of the approximating architecture, yielding a better approximation error. Furthermore, different regularity assumptions are required, in \cite{belomestny23} a form Hölder-regularity whereas Theorem \ref{theorem_lu_param_growth} requires smoothness of certain degree. Moreover, a key difference is that the result in \cite{belomestny23} is based on the ReQU activation function $\sigma^{\text{ReQU}}(x) = (x\lor 0)^2$ which is capable of approximating several higher order derivatives simultaneously.
	The approximation in \cite{belomestny23} is essentially based on tensor-product splines, which are certain piecewise polynomials. The crucial point here is that the coefficients of corresponding normalized basis splines are a priori uniformly bounded in terms of the approximated function.\\
	The choice of the ReQU in \cite{belomestny23} is therefore crucial as it can represent piecewise polynomials exactly, thus, in particular also the identity mapping and products. Expanding the approximating architecture yields that the parameters, which are contained in a compact set, can be restricted to the interval $[-1,1]$, respectively.\\
	The usage of ReQU in \cite{belomestny23} is essential, as for the ReLU activation function, used in \cite{lu_main}, approximability is only achievable in spaces $W^{q,p}([0,1]^d)$ for $0\leq q\leq 1$ and $1\leq p\leq \infty$ (see \cite{guehring20}) due to the first order irregularity of the ReLU in zero. In addition, the ReLU is only capable of approximating multiplications with decreasing error and increasing architectures (see \cite[Lemma 4.2]{lu_main}) as opposed to the ReQU.
	\medskip
	
	Finally, we compare the result in Theorem \ref{theorem_lu_param_growth} to \cite[Proposition 4.8]{raslan21} summarized in Table \ref{comparison:2}. Again we consider the approximation of some $f\in \mathcal{C}^q([0,1]^d)$. As a consequence, the result in \cite{raslan21} yields that for $\epsilon>0$ and the number of nonzero weights being of order $\mathcal{O}(\epsilon^{-d/q})$, the approximation error equals $\epsilon$ and the norm of the parameters is of order $\mathcal{O}(\epsilon^{-(2+3d/q)})$. Thus, if the number of nonzero weights is of order $\mathcal{O}(W)$ then the approximation error is of complexity $\mathcal{O}(W^{-q/d})$ and the parameters of order $\mathcal{O}(W^{3+2q/d})$.

The approximation result of Theorem \ref{theorem_lu_param_growth} in Table \ref{comparison:1} is formulated in terms of width $N$ and depth $L$. The overall number of weights is of order $\mathcal{O}(N^2L)$, since each of the $\mathcal{O}(L)$ layers is defined by an affine map with at most $N^2+N$ parameters. Note that the construction of \cite{lu_main} yields fully connected networks and no sparsity of the realizing parameters is claimed. The bound $\mathcal{O}(N^2L)$ therefore also serves as an upper bound for the number of nonzero weights, so that the comparison with \cite{raslan21} is carried out at a weight budget that is not favourable to Theorem \ref{theorem_lu_param_growth}. Identifying the product $N^2L$ by $W$, then, e.g., under $N\approx L$ (and hence $N,L\approx W^{1/3}$) we derive that the growth of the norm of the parameters in Theorem \ref{theorem_lu_param_growth} given in Table \ref{comparison:1} is of order $\mathcal{O}(W^{\frac{9q-4}{3d}\lor \frac{5}{6}})$. As a consequence, for $18q\leq 5d+8$ the upper bound of the norm of the parameters in Theorem \ref{theorem_lu_param_growth} grows slower than that of \cite{raslan21}. In case $18q>5d+8$ this applies only if $3q\leq 9d+4$.

The last row of Table \ref{comparison:2} contains the partition of unity based approximation of \cite[Theorem 4.1]{guehring20}, the growth of whose parameters is derived in Appendix \ref{sec:pou}. At the normalized number of nonzero weights it grows slower than the one of Theorem \ref{theorem_lu_param_growth} for every $q$ and $d$. The two results are, however, not interchangeable. 

In Theorem \ref{theorem_lu_param_growth} the width and depth are prescribed, and the approximation error is of order $(NL)^{-2q/d}$, up to logarithmic factors. This error estimate depends on $N$ and $L$ through their product. However, the controlled-parameter construction in Theorem \ref{theorem_lu_param_growth} has depth $\mathcal{O}(L^2 \log L)$, rather than the $\mathcal{O}(L \log L)$ depth in Theorem~\ref{lu_param_growth}. Thus, the product-form error estimate alone does not establish the same width-depth trade-off, or joint width-depth optimality, for the modified construction. In \cite[Theorem 4.1]{guehring20}, no width-depth trade-off is available.

Moreover, as above, the comparison is carried out at a weight budget that is not favourable to Theorem \ref{theorem_lu_param_growth}, since the networks of \cite{lu_main} are counted by $\mathcal{O}(N^2L)$.
\medskip
	
In summary, compared to state of the art results, except for \cite{belomestny23} where uniform boundedness is achieved by using ReQU activations and for \cite{guehring20} where no trade-off between width and depth is available, the bound on the norm of the parameters of Theorem \ref{theorem_lu_param_growth} grows slower in most cases, depending on the input dimension and regularity of the approximated function.

	\section{Growth of parameters of approximating neural networks}
	\label{sec:param_growth}
	 In this section we provide analytical results on the asymptotic behavior of the supremum norm of the parameters of two approximating fully connected feed forward neural network architectures based on \cite{lu_main}. 
For completeness, we first provide the definition of a fully connected feed forward neural network.
\begin{definition}
	\label{def:network}
	Given $L\in \mathbb{N}$ and $n_l\in \mathbb{N}$ for $0\leq l\leq L$, a fully connected feed forward neural network $\mathcal{N}_\theta$ with activation function $\sigma$ is defined as \[\mathcal{N}_\theta= L_{\theta_L}\circ\dots\circ L_{\theta_1}\] for $L_{\theta_l}:\mathbb{R}^{n_{l-1}}\to\mathbb{R}^{n_l}$ with $L_{\theta_l}(z):=\sigma(w^lz+\beta^l)$ for $1\leq l\leq L-1$ and $L_{\theta_L}(z):= w^Lz+\beta^L$ where $\theta_l=(w^l,\beta^l)$ with $w^l\in \mathcal{L}(\mathbb{R}^{n_{l-1}},\mathbb{R}^{n_l})\simeq\mathbb{R}^{n_l\times n_{l-1}}$, $\beta^l\in \mathbb{R}^{n_l}$ for $1\leq l\leq L$. Further we define the depth of the network by $\mathcal{D}(\mathcal{N})=L$, the width by $\mathcal{W}(\mathcal{N})=N= \max_l n_l$ and denote by $\mathcal{FNN}$ the class of fully connected feed forward neural networks.
\end{definition}

The main goal here is to study the following problem: Given $f \in \mathcal{X}$, e.g., $\mathcal{X}=\mathcal{C}^q([0,1]^d)$, assume that there exists $c>0$
such that for all $N, L\in \mathbb{N}$ there exists $\phi_{N,L}\in\mathcal{FNN}$ with width $\mathcal{W}(\phi_{N,L})=w(N)$ and depth $\mathcal{D}(\phi_{N,L})=d(L)$ fulfilling
\begin{align}
	\label{approximation}
	\Vert \phi_{N,L}-f\Vert_\mathcal{Y}\leq c \Vert f\Vert_\mathcal{X} \alpha_\mathcal{X}(N,L)
\end{align}
with  $\mathcal{X}\hookrightarrow \mathcal{Y}$, $\alpha:\mathbb{N}^2\to [0,\infty)$ monotonically decreasing in both components, and decreasing to zero in at least one component.
In case the constant $c$ is independent of $f, N$ and $L$ we write \[\Vert \phi_{N,L}-f\Vert_\mathcal{Y}\lesssim\Vert f\Vert_\mathcal{X} \alpha_\mathcal{X}(N,L)\] for \eqref{approximation}.
Given this kind of approximation result, the question is how $\|\theta_{\phi_{N,L}}\|_\infty$, i.e., the supremum norm of the parameters $\theta_{\phi_{N,L}}$ of $\phi_{N,L}$, behaves asymptotically for $\phi_{N,L}$ approximating $f$ via the width $N$ and/or the depth $L$ going to infinity.

Note that, in contrast to most works that deal with approximability in terms of width and depth of neural networks, we take such results as given and rather focus on the worst-case growth of the supremum of the parameters of the approximating neural network. Furthermore, note that the asymptotic growth of the network parameters highly depends on the way the neural-network approximation is constructed.

An important result in this context is provided in \cite{lu_main}, which shows via an explicit construction, that fully connected feed forward neural networks with ReLU activation functions, with width of order $N\log(8N)$ and depth of order $L\log(4L)$, can approximate functions $f\in \mathcal{C}^q([0,1]^d)$ with an error of order $\Vert f\Vert_{\mathcal{C}^q([0,1]^d)}(NL)^{-2q/d}$. That is in \eqref{approximation} it holds $\mathcal{X}= \mathcal{C}^q([0,1]^d), \mathcal{Y}=L^\infty([0,1]^d)$ and $\alpha_\mathcal{X}(N,L)=(NL)^{-2q/d}$. The main result in \cite{lu_main} reads as follows
	
	\begin{theorem}{\cite[Theorem 1.1]{lu_main}}
		\label{lu_param_growth}
		For $f\in \mathcal{C}^q([0,1]^d)$ with $q\in\mathbb{N}^+$ there exists some ReLU generated neural network $\phi$ with width $\mathcal{W}(\phi)=\mathcal{O}(N\log N)$ and depth $\mathcal{D}(\phi)=\mathcal{O}(L\log L)$ such that
		\[
		\Vert f-\phi\Vert_{L^\infty([0,1]^d)}\lesssim \Vert f\Vert_{\mathcal{C}^q([0,1]^d)}N^{-2q/d}L^{-2q/d}.
		\]
	\end{theorem}

	Based on this approximation, the main result of this section is as follows.
	\begin{theorem}
		\label{theorem_lu_param_growth}
		Let $d,q\in \mathbb{N}$ and $f\in \mathcal{C}^q([0,1]^d)$ with $\tilde{L}$ denoting the Lipschitz constant of $f$ on $[0,1]^d$. Then for any $N,L\in \mathbb{N}$ there exists a ReLU feed forward neural network $f_{N,L}$ with width $C_1N\log(8N)$ and depth $C_2L^2\log(4L)$ such that
		\[
		\Vert f-f_{N,L}\Vert_{L^\infty([0,1]^d)}\lesssim\Vert f\Vert_{\mathcal{C}^q([0,1]^d)}N^{-2q/d}L^{-2q/d},
		\]
		with $C_1,C_2>0$ independent of $f, N$ and $L$, and such that
		\[\|\theta_{f_{N,L}}\|_\infty \lesssim \max(d^3\tilde{L}^3N^{(6q-3)/d}L^{(6q-2)/d}, NL(N+L^2)),\]
		where $\theta_{f_{N,L}}$ are the parameters of $f_{N,L}$.
	\end{theorem}
	The approximation-error bound in Theorem \ref{theorem_lu_param_growth} has the same product form in $N$ and $L$ as the bound in Theorem \ref{lu_param_growth}. However, the modification used to obtain polynomial growth of the parameters increases the depth from $\mathcal{O}(L \log L)$ to $\mathcal{O}(L^2 \log L)$. Thus, the theorem allows $N$ and $L$ to be prescribed independently, but does not establish joint width-depth optimality as in Theorem \ref{lu_param_growth} for the modified architecture in Theorem \ref{theorem_lu_param_growth}.
	
	We refer to Table \ref{comparison:1} for the growth of the norm of parameters realizing the approximating networks at normalized width. In order to prove Theorem \ref{theorem_lu_param_growth}, we follow the construction for the proof of Theorem \ref{lu_param_growth}, which is based on the two main auxiliary results \cite[Theorem 2.1 and Theorem 2.2]{lu_main} as follows.
\begin{enumerate}[leftmargin=0.5cm]
\item The result in \cite[Theorem 2.2]{lu_main} gives a constructive proof for approximating a given, sufficiently regular function on $[0,1]^d$ with an approximation error of order $(NL)^{-2q/d}$ with a ReLU-neural-network with width $\mathcal{O}(N\log N)$ and depth $\mathcal{O}(L\log L)$ outside a trifling region
	\begin{align}
		\label{trifling_region}
		\Omega([0,1]^d, R,\delta):=\bigcup_{i=1}^d\bigg\{x\in [0,1]^d: x_i\in \bigcup_{k=1}^{R-1}\bigg(\frac{k}{R}-\delta, \frac{k}{R}\bigg)\bigg\}.
	\end{align}
\item The result \cite[Theorem 2.1]{lu_main} then shows how such an approximation can be extended to approximate the function on all of $[0,1]^d$.
\item Finally, in the main result \cite[Theorem 1.1]{lu_main}, the trifling region is chosen small enough, which infers the final asymptotic behavior of the network parameters in terms of $N$ and $L$.
\end{enumerate}	
Accordingly, our proof of Theorem \ref{theorem_lu_param_growth} is divided into three subsections corresponding to the steps 1. - 3. above, where the main effort lies in the first step.

\subsection[Estimate 1]{Estimating $\|\theta_\phi\|_\infty$ for $\phi$ in \cite[Theorem 2.2]{lu_main}:} The approximating neural network of \cite[Theorem 2.2]{lu_main} is given by
	\[
		\phi(x):=\sum_{\Vert\alpha\Vert_1\leq q-1}\varphi(\frac{1}{\alpha!}\phi_\alpha(\Psi(x)), P_\alpha(x-\Psi(x)))
	\]
	for $x\in \mathbb{R}^d$, where the role of the subnetworks $\Psi, P_\alpha, \phi_\alpha, \varphi$ is as follows:
	\begin{itemize}[leftmargin=0.5cm]
		\item The ReLU FNN $\Psi$ realizes projections of subcubes of $[0,1]^d$ to exactly one corner of the subcube based on one-dimensional step functions $\psi$ (see considerations on \cite[Proposition 4.3]{lu_main} for $\psi$)
		\item The ReLU FNN $P_\alpha$ achieves an approximation of multinomials of order at most $q-1$ (see \cite[Proposition 4.1]{lu_main}).
		\item The ReLU FNN $\phi_\alpha$ achieves fitting partial derivatives of $f$ of order at most $q-1$ at the corners of the subcubes to which $\Psi$ projects to (see \cite[Proposition~4.4]{lu_main}).
		\item The ReLU FNN $\varphi$ approximates binomials (see \cite[Lemma 4.2]{lu_main}).
	\end{itemize}
The approximation error for $\phi$ as above is estimated in \cite[Step 3, p. 25 ff]{lu_main}, which relies on the triangle inequality and the approximation properties of the single components.
Figure \ref{fig:overview} provides an overview of the relevant subresults derived in \cite{lu_main,shen_external1,shen_external2} for constructing $\phi$ as above. In the following we will not give and explain each of these subresults in detail, but rather refer to the original references.

\tikzset{every picture/.style={line width=0.75pt}}
\tikzset{every picture/.style={line width=0.75pt}}
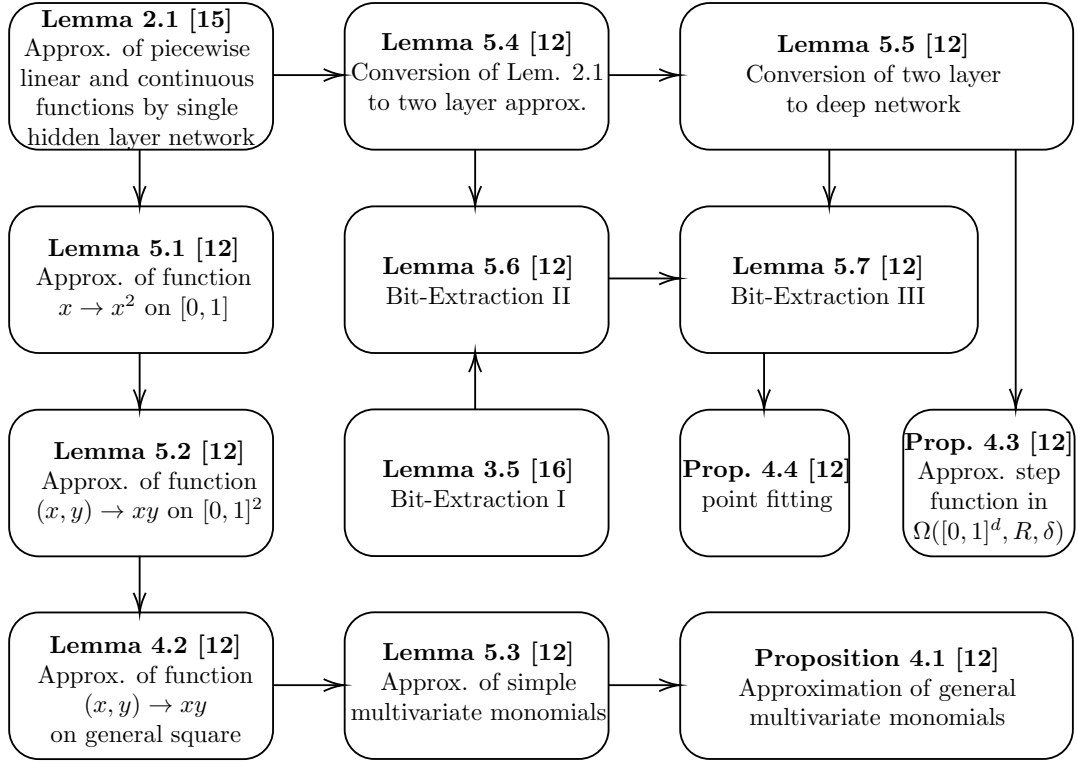
\begin{figure}[h!]
		\centering
		\tikzset{every picture/.style={line width=0.75pt}}
		\scalebox{.93}{
		\begin{tikzpicture}[x=0.75pt,y=0.75pt,yscale=-1,xscale=1]
			
			\draw   (10,26.87) .. controls (10,18.1) and (17.1,11) .. (25.87,11) -- (135.72,11) .. controls (144.48,11) and (151.58,18.1) .. (151.58,26.87) -- (151.58,74.47) .. controls (151.58,83.23) and (144.48,90.33) .. (135.72,90.33) -- (25.87,90.33) .. controls (17.1,90.33) and (10,83.23) .. (10,74.47) -- cycle ;
			\draw   (10,136.87) .. controls (10,128.1) and (17.1,121) .. (25.87,121) -- (135.72,121) .. controls (144.48,121) and (151.58,128.1) .. (151.58,136.87) -- (151.58,184.47) .. controls (151.58,193.23) and (144.48,200.33) .. (135.72,200.33) -- (25.87,200.33) .. controls (17.1,200.33) and (10,193.23) .. (10,184.47) -- cycle ;
			\draw   (10,246.87) .. controls (10,238.1) and (17.1,231) .. (25.87,231) -- (135.72,231) .. controls (144.48,231) and (151.58,238.1) .. (151.58,246.87) -- (151.58,294.47) .. controls (151.58,303.23) and (144.48,310.33) .. (135.72,310.33) -- (25.87,310.33) .. controls (17.1,310.33) and (10,303.23) .. (10,294.47) -- cycle ;
			\draw   (10,356.53) .. controls (10,347.77) and (17.1,340.67) .. (25.87,340.67) -- (135.72,340.67) .. controls (144.48,340.67) and (151.58,347.77) .. (151.58,356.53) -- (151.58,404.13) .. controls (151.58,412.9) and (144.48,420) .. (135.72,420) -- (25.87,420) .. controls (17.1,420) and (10,412.9) .. (10,404.13) -- cycle ;
			\draw   (190,26.87) .. controls (190,18.1) and (197.1,11) .. (205.87,11) -- (315.72,11) .. controls (324.48,11) and (331.58,18.1) .. (331.58,26.87) -- (331.58,74.47) .. controls (331.58,83.23) and (324.48,90.33) .. (315.72,90.33) -- (205.87,90.33) .. controls (197.1,90.33) and (190,83.23) .. (190,74.47) -- cycle ;
			\draw   (190,136.87) .. controls (190,128.1) and (197.1,121) .. (205.87,121) -- (315.72,121) .. controls (324.48,121) and (331.58,128.1) .. (331.58,136.87) -- (331.58,184.47) .. controls (331.58,193.23) and (324.48,200.33) .. (315.72,200.33) -- (205.87,200.33) .. controls (197.1,200.33) and (190,193.23) .. (190,184.47) -- cycle ;
			\draw   (190,246.7) .. controls (190,237.94) and (197.1,230.83) .. (205.87,230.83) -- (315.72,230.83) .. controls (324.48,230.83) and (331.58,237.94) .. (331.58,246.7) -- (331.58,294.3) .. controls (331.58,303.06) and (324.48,310.17) .. (315.72,310.17) -- (205.87,310.17) .. controls (197.1,310.17) and (190,303.06) .. (190,294.3) -- cycle ;
			\draw   (190,356.45) .. controls (190,347.69) and (197.1,340.58) .. (205.87,340.58) -- (315.72,340.58) .. controls (324.48,340.58) and (331.58,347.69) .. (331.58,356.45) -- (331.58,404.05) .. controls (331.58,412.81) and (324.48,419.92) .. (315.72,419.92) -- (205.87,419.92) .. controls (197.1,419.92) and (190,412.81) .. (190,404.05) -- cycle ;
			\draw   (370,26.87) .. controls (370,18.1) and (377.1,11) .. (385.87,11) -- (564.97,11) .. controls (573.73,11) and (580.83,18.1) .. (580.83,26.87) -- (580.83,74.47) .. controls (580.83,83.23) and (573.73,90.33) .. (564.97,90.33) -- (385.87,90.33) .. controls (377.1,90.33) and (370,83.23) .. (370,74.47) -- cycle ;
			\draw   (370,136.87) .. controls (370,128.1) and (377.1,121) .. (385.87,121) -- (513.8,121) .. controls (522.56,121) and (529.67,128.1) .. (529.67,136.87) -- (529.67,184.47) .. controls (529.67,193.23) and (522.56,200.33) .. (513.8,200.33) -- (385.87,200.33) .. controls (377.1,200.33) and (370,193.23) .. (370,184.47) -- cycle ;
			\draw   (370,246.87) .. controls (370,238.1) and (377.1,231) .. (385.87,231) -- (444.8,231) .. controls (453.56,231) and (460.67,238.1) .. (460.67,246.87) -- (460.67,294.47) .. controls (460.67,303.23) and (453.56,310.33) .. (444.8,310.33) -- (385.87,310.33) .. controls (377.1,310.33) and (370,303.23) .. (370,294.47) -- cycle ;
			\draw   (489.67,246.87) .. controls (489.67,238.1) and (496.77,231) .. (505.53,231) -- (565.72,231) .. controls (574.48,231) and (581.58,238.1) .. (581.58,246.87) -- (581.58,294.47) .. controls (581.58,303.23) and (574.48,310.33) .. (565.72,310.33) -- (505.53,310.33) .. controls (496.77,310.33) and (489.67,303.23) .. (489.67,294.47) -- cycle ;
			\draw   (370,356.45) .. controls (370,347.69) and (377.1,340.58) .. (385.87,340.58) -- (564.8,340.58) .. controls (573.56,340.58) and (580.67,347.69) .. (580.67,356.45) -- (580.67,404.05) .. controls (580.67,412.81) and (573.56,419.92) .. (564.8,419.92) -- (385.87,419.92) .. controls (377.1,419.92) and (370,412.81) .. (370,404.05) -- cycle ;
			\draw    (151.67,50) -- (187.67,50) ;
			\draw [shift={(189.67,50)}, rotate = 180] [color={rgb, 255:red, 0; green, 0; blue, 0 }  ][line width=0.75]    (10.93,-3.29) .. controls (6.95,-1.4) and (3.31,-0.3) .. (0,0) .. controls (3.31,0.3) and (6.95,1.4) .. (10.93,3.29)   ;
			\draw    (151.67,380) -- (187.67,380) ;
			\draw [shift={(189.67,380)}, rotate = 180] [color={rgb, 255:red, 0; green, 0; blue, 0 }  ][line width=0.75]    (10.93,-3.29) .. controls (6.95,-1.4) and (3.31,-0.3) .. (0,0) .. controls (3.31,0.3) and (6.95,1.4) .. (10.93,3.29)   ;
			\draw    (331.67,380) -- (367.67,380) ;
			\draw [shift={(369.67,380)}, rotate = 180] [color={rgb, 255:red, 0; green, 0; blue, 0 }  ][line width=0.75]    (10.93,-3.29) .. controls (6.95,-1.4) and (3.31,-0.3) .. (0,0) .. controls (3.31,0.3) and (6.95,1.4) .. (10.93,3.29)   ;
			\draw    (331.67,160) -- (367.67,160) ;
			\draw [shift={(369.67,160)}, rotate = 180] [color={rgb, 255:red, 0; green, 0; blue, 0 }  ][line width=0.75]    (10.93,-3.29) .. controls (6.95,-1.4) and (3.31,-0.3) .. (0,0) .. controls (3.31,0.3) and (6.95,1.4) .. (10.93,3.29)   ;
			\draw    (331.67,50) -- (367.67,50) ;
			\draw [shift={(369.67,50)}, rotate = 180] [color={rgb, 255:red, 0; green, 0; blue, 0 }  ][line width=0.75]    (10.93,-3.29) .. controls (6.95,-1.4) and (3.31,-0.3) .. (0,0) .. controls (3.31,0.3) and (6.95,1.4) .. (10.93,3.29)   ;
			\draw    (80,90) -- (80,119) ;
			\draw [shift={(80,121)}, rotate = 270] [color={rgb, 255:red, 0; green, 0; blue, 0 }  ][line width=0.75]    (10.93,-3.29) .. controls (6.95,-1.4) and (3.31,-0.3) .. (0,0) .. controls (3.31,0.3) and (6.95,1.4) .. (10.93,3.29)   ;
			\draw    (80,200) -- (80,229) ;
			\draw [shift={(80,231)}, rotate = 270] [color={rgb, 255:red, 0; green, 0; blue, 0 }  ][line width=0.75]    (10.93,-3.29) .. controls (6.95,-1.4) and (3.31,-0.3) .. (0,0) .. controls (3.31,0.3) and (6.95,1.4) .. (10.93,3.29)   ;
			\draw    (80,310) -- (80,339) ;
			\draw [shift={(80,341)}, rotate = 270] [color={rgb, 255:red, 0; green, 0; blue, 0 }  ][line width=0.75]    (10.93,-3.29) .. controls (6.95,-1.4) and (3.31,-0.3) .. (0,0) .. controls (3.31,0.3) and (6.95,1.4) .. (10.93,3.29)   ;
			\draw    (260,90) -- (260,119) ;
			\draw [shift={(260,121)}, rotate = 270] [color={rgb, 255:red, 0; green, 0; blue, 0 }  ][line width=0.75]    (10.93,-3.29) .. controls (6.95,-1.4) and (3.31,-0.3) .. (0,0) .. controls (3.31,0.3) and (6.95,1.4) .. (10.93,3.29)   ;
			\draw    (450,90) -- (450,119) ;
			\draw [shift={(450,121)}, rotate = 270] [color={rgb, 255:red, 0; green, 0; blue, 0 }  ][line width=0.75]    (10.93,-3.29) .. controls (6.95,-1.4) and (3.31,-0.3) .. (0,0) .. controls (3.31,0.3) and (6.95,1.4) .. (10.93,3.29)   ;
			\draw    (415,200) -- (415,229) ;
			\draw [shift={(415,231)}, rotate = 270] [color={rgb, 255:red, 0; green, 0; blue, 0 }  ][line width=0.75]    (10.93,-3.29) .. controls (6.95,-1.4) and (3.31,-0.3) .. (0,0) .. controls (3.31,0.3) and (6.95,1.4) .. (10.93,3.29)   ;
			\draw    (550,90) -- (550,228.67) ;
			\draw [shift={(550,230.67)}, rotate = 270] [color={rgb, 255:red, 0; green, 0; blue, 0 }  ][line width=0.75]    (10.93,-3.29) .. controls (6.95,-1.4) and (3.31,-0.3) .. (0,0) .. controls (3.31,0.3) and (6.95,1.4) .. (10.93,3.29)   ;
			\draw    (260,230.67) -- (260,201.67) ;
			\draw [shift={(260,199.67)}, rotate = 90] [color={rgb, 255:red, 0; green, 0; blue, 0 }  ][line width=0.75]    (10.93,-3.29) .. controls (6.95,-1.4) and (3.31,-0.3) .. (0,0) .. controls (3.31,0.3) and (6.95,1.4) .. (10.93,3.29)   ;

			\draw (7,11) node [anchor=north west][inner sep=0.75pt]  [font=\footnotesize]  {$ \begin{array}{c}
					{\displaystyle \textbf{Lemma\ 2.1\ \cite{shen_external1}}}\\
					\text{Approx. of piecewise}\\
					\text{linear and continuous}\\
					\text{functions by single}\\
					\text{ hidden layer network}
				\end{array}$};
			
			\draw (16,135) node [anchor=north west][inner sep=0.75pt]  [font=\footnotesize]  {$ \begin{array}{c}
					{\displaystyle \textbf{Lemma\ 5.1\ \cite{lu_main}}}\\
					\text{Approx. of function}\\
					x\rightarrow x^2 \ \text{on} \ [0,1]
				\end{array}$};
			
			\draw (16,245) node [anchor=north west][inner sep=0.75pt]  [font=\footnotesize]  {$ \begin{array}{c}
					{\displaystyle \textbf{Lemma\ 5.2\ \cite{lu_main}}}\\
					\text{Approx. of function}\\
					(x,y)\rightarrow xy \ \text{on} \ [0,1]^2
				\end{array}$};
			
			\draw (17,350) node [anchor=north west][inner sep=0.75pt]  [font=\footnotesize]  {$ \begin{array}{c}
					{\displaystyle \textbf{Lemma\ 4.2\ \cite{lu_main}}}\\
					\text{Approx. of function}\\
					(x,y)\rightarrow xy\\
					\text{on general square}
				\end{array}$};

			\draw (185,354) node [anchor=north west][inner sep=0.75pt]  [font=\footnotesize]  {$ \begin{array}{c}
					{\displaystyle \textbf{Lemma\ 5.3\ \cite{lu_main}}}\\
					\text{Approx. of simple}\\
					\text{multivariate monomials}
				\end{array}$};
			
			\draw (393,357) node [anchor=north west][inner sep=0.75pt]  [font=\footnotesize]  {$ \begin{array}{c}
					{\displaystyle \textbf{Proposition\ 4.1\ \cite{lu_main}}}\\
					\text{Approximation of general}\\
					\text{multivariate monomials}
				\end{array}$};
			
			\draw (185,24) node [anchor=north west][inner sep=0.75pt]  [font=\footnotesize]  {$ \begin{array}{c}
					{\displaystyle \textbf{Lemma\ 5.4\ \cite{lu_main}}}\\
					\text{Conversion of Lem. 2.1}\\
					\text{to two layer approx.}
				\end{array}$};

			\draw (397,25) node [anchor=north west][inner sep=0.75pt]  [font=\footnotesize]  {$ \begin{array}{c}
					{\displaystyle \textbf{Lemma\ 5.5\ \cite{lu_main}}}\\
					\text{Conversion of two layer}\\
					\text{to deep network}
				\end{array}$};

			\draw (202,145) node [anchor=north west][inner sep=0.75pt]  [font=\footnotesize]  {$ \begin{array}{c}
					{\displaystyle \textbf{Lemma\ 5.6\ \cite{lu_main}}}\\
					\text{Bit-Extraction II}
				\end{array}$};
			
			\draw (202,255) node [anchor=north west][inner sep=0.75pt]  [font=\footnotesize]  {$ \begin{array}{c}
					{\displaystyle \textbf{Lemma\ 3.5\ \cite{shen_external2}}}\\
					\text{Bit-Extraction I}
				\end{array}$};
			
			\draw (389,145) node [anchor=north west][inner sep=0.75pt]  [font=\footnotesize]  {$ \begin{array}{c}
					{\displaystyle \textbf{Lemma\ 5.7\ \cite{lu_main}}}\\
					\text{Bit-Extraction III}
				\end{array}$};
			
			\draw (363,255) node [anchor=north west][inner sep=0.75pt]  [font=\footnotesize]  {$ \begin{array}{c}
					{\displaystyle \textbf{Prop.\ 4.4\ \cite{lu_main}}}\\
					\text{point fitting}
				\end{array}$};
			
			\draw (482,240) node [anchor=north west][inner sep=0.75pt]  [font=\footnotesize]  {$ \begin{array}{c}
					{\displaystyle \textbf{Prop.\ 4.3\ \cite{lu_main}}}\\
					\text{Approx. step}\\
					\text{function in}\\
					\Omega([0,1]^d, R,\delta)
				\end{array}$};

		\end{tikzpicture}
	}\caption{Overview of structure of results in \cite{lu_main}}
		\label{fig:overview}
	\end{figure}
	Now we estimate the growth of the norm of the parameters for each of these subresults to finally estimate $\|\theta_\phi\|_\infty$. In accordance with the chain of dependencies depicted in Figure \ref{fig:overview}, we start by analyzing \cite[Lemma 2.1]{shen_external1}.
	\paragraph{Complexity estimation in \cite[Lemma 2.1]{shen_external1}:} 
	The result in \cite[Lemma 2.1]{shen_external1} shows that the set of continuous piecewise linear functions with $N$ pieces mapping an interval to $\mathbb{R}$ is expressible by a single hidden layer ReLU network $\phi$ of width $\mathcal{W}(\phi)=\mathcal{O}(N)$. We restrict ourselves to two cases that occur in the subsequent results in \cite{lu_main, shen_external1, shen_external2}, i.e., the general estimation of the growth of the parameters of the interpolating networks is not necessary, covered by the next two lemmata.
	\begin{lemma}
		\label{lemm:equi_1}
		Let $\tilde{R}\in\mathbb{N}$, $R>0$ and $y_k\in \mathbb{R}$ for $0\leq k\leq \tilde{R}$ be given. Then there exists a single hidden layer ReLU network $\phi$ with width $\mathcal{W}(\phi)=\tilde{R}$ such that $\phi(x_k)=y_k$ where $x_k = k/R$ for $0\leq k\leq \tilde{R}$, i.e., the $x_k$ are equidistantly distributed in the interval $[0,\tilde{R}/R]$. Furthermore, the network $\phi$ fulfills
		\[
		\|\theta_\phi\|_\infty \leq 
		\max(1,\vert x_0\vert, \vert x_{\tilde{R}-1}\vert, \vert y_0\vert, R \max_j\vert (y_{j+2}-y_{j+1})-(y_{j+1}-y_j)\vert).
		\]
	\end{lemma}
	\begin{proof}
		The network $\phi$ may be realized by $\phi(x)=W_2\relu(W_1x+b_1)+b_2$ where by \cite[Lemma 2.1]{shen_external1} the parameters $W_1, b_1, b_2$ can be chosen as\vspace*{-0.1cm} \[W_1=(1,\dots, 1)^T\in \mathbb{R}^{\tilde{R}\times 1}, b_1 = (-x_0, \dots, -x_{\tilde{R}-1})^T.\] It remains to determine $b_2$ and $W_2$ and the asymptotical behavior of the parameters of $\phi$. It is straightforward to show with $b_2 = y_0$ and for $W_2 = (w_0, \dots, w_{\tilde{R}-1})$ with
\[
w_0 =\frac{y_1-y_0}{x_1-x_0} = R(y_1-y_0),\vspace*{-0.2cm}
\]
		\begin{equation}
			\label{arithmet_rec}
			w_j = \frac{1}{x_{j+1}-x_j}(y_{j+1}-y_0-\sum_{l=0}^{j-1}w_l(x_{j+1}-x_l)) = R(y_{j+1}-y_0)-\sum_{l=0}^{j-1}w_l(j-l+1) 
		\end{equation}
		for $1 \leq j  \leq \tilde{R}-1$, that $\phi(x_k)=y_k$ for $0\leq k \leq \tilde{R}$ holds. 
		By resolving the arithmetic recursion \eqref{arithmet_rec} one can show that the unique solution of \eqref{arithmet_rec} is given by
		\begin{align*}
			w_j = R((y_{j+1}-y_j)-(y_j-y_{j-1})) ~ ~ \text{for} ~ ~ 1\leq j\leq \tilde{R}-1.
		\end{align*}
		This may be proven by induction. The induction start follows by
		\[w_1 = R(y_2-y_0)-2w_0 = R((y_2-y_1)-(y_1-y_0)).\]
		
		The induction step is a consequence of\vspace*{-0.4cm}
		\begin{align*}
			 \small~ w_{j+1}&=R(y_{j+2}-y_0) -\sum_{l=0}^j w_l (j-l+2)\\[-0.1cm]
			&=R(y_{j+2}-y_0)-R\sum_{l=1}^j(y_{l+1}-2y_l+y_{l-1})(j-l+2)-R(j+2)(y_1-y_0)\\[-0.1cm]
			&= R(y_{j+2}-y_0)-R\sum_{l=2}^{j-1}y_l(j-l+3-2(j-l+2)+j-l+1)\\[-0.1cm]
			&\hspace{0.5cm}-R(3y_j+2y_{j+1}-4y_j-2(j+1)y_1+(j+1)y_0+jy_1)-R(j+2)(y_1-y_0)\\
			&= R(y_{j+2}-2y_{j+1}+y_j).
		\end{align*}
		\normalsize
		Thus, we derive that
		\[
			\|\theta_\phi\|_\infty \leq 
		\max(1,\vert x_0\vert, \vert x_{\tilde{R}-1}\vert, \vert y_0\vert, R \max_j\vert (y_{j+2}-y_{j+1})-(y_{j+1}-y_j)\vert).\qedhere
		\]
	\end{proof}
	Note that by employing the triangle inequality together with
	\begin{align}
		\label{def:xy}
		X := \max_{0\leq k\leq \tilde{R}}\vert x_k\vert ~ ~ ~ \text{ and } ~ ~ ~  Y:=\vert y_0\vert+\max_{0\leq k\leq \tilde{R}-1}\vert y_{k+1}-y_k\vert
	\end{align} we obtain in the previous Lemma the bound $\|\theta_\phi\|_\infty \lesssim\max(X,RY)$.
	
	Next we consider a similar result generalized to certain inequidistant grids $\{x_k\}_k$ of the interval $[0, \tilde{R}/R]$. Note that the following Lemma is general enough to cover the instances in \cite{lu_main} used by \cite[Lemma 2.1]{shen_external1}.
	\begin{lemma}
		\label{lemm:inequi_1}
		Let $\tilde{R}\in\mathbb{N}$, $R>0$ and $\delta =\frac{1}{(c+1)R}$ for some $c\in \mathbb{N}$ be such that there exist $m,n\in \mathbb{N}$ with $2\tilde{R}=m(n+1)$ where $n+1=2p$ for some $p\in \mathbb{N}$. Let further the grid points $\{x_k\}_k$ be given as\vspace*{-0.2cm}
		\[ x_{2k}=\frac{k}{R} ~ ~ \text{ for } ~ ~  0\leq k\leq \tilde{R}\quad \text{and}\quad x_{2k-1}=\frac{k}{R}-\delta ~ ~ \text{ for } ~ ~ 1\leq k\leq \tilde{R}.\vspace*{-0.2cm}
		\] Assume that $y_k\in \mathbb{R}$ for $0\leq k\leq 2\tilde{R}$ are given. Then there exists a single hidden layer ReLU network $\phi$ with width $\mathcal{W}(\phi)=2m$ such that \[\phi(x_{j(n+1)})=y_{j(n+1)}\quad \text{for} \quad j=0,\dots, m\] \[\text{and}\quad \phi(x_{j(n+1)+n})=y_{j(n+1)+n}\quad \text{for} \quad j=0,\dots, m-1.\] \pagebreak[2]
		
		Furthermore, the network $\phi$ fulfills
		\[
		\|\theta_\phi\|_\infty \lesssim\max(X, cRY),\vspace*{-0.2cm}
		\]
			\begin{align}
			\label{def:xy2}
			\text{where } ~ ~ ~ X := \max_{0\leq k\leq 2\tilde{R}}\vert x_k\vert ~ ~ ~ \text{ and } ~ ~ ~  Y:=\vert y_0\vert+\max_{0\leq k\leq 2\tilde{R}-1}\vert y_{k+1}-y_k\vert.
		\end{align}
	\end{lemma}

	\begin{proof}
	We recall that interpolation is considered only in the points $x_{j(n+1)}$ for $j=0,\dots, m$ and $x_{j(n+1)+n}$ for $j=0,\dots, m-1$. Define
	\begin{align*}
		&z_{2j}=x_{j(n+1)}, ~ v_{2j} = y_{j(n+1)}, ~ ~ ~ \text{for } ~ j=0, \dots, m\\
		&z_{2j-1}=x_{j(n+1)-1}, ~ v_{2j-1} = y_{j(n+1)-1}, ~ ~ ~ \text{for } ~ j=1, \dots, m.´
	\end{align*} 
	Then the network $\phi$, realized in form of $\phi(x)=W_2\relu(W_1x+b_1)+b_2$, has to fulfill the conditions $\phi(z_k)=v_k$ for $0\leq k\leq 2m$. The parameters $W_1, b_1, b_2$ may be chosen similarly to Lemma \ref{lemm:equi_1} by $W_1 = (1,\dots, 1)^T\in \mathbb{R}^{2m}$, $b_1 = (-z_0, \dots, -z_{2m-1})^T$ and $b_2 = y_0$. 
	Then $W_2 = (w_0, \dots, w_{2m-1})^T$ is the unique solution of the system
	\begin{align}
		\label{lin_sys}
		ZW_2 = V
	\end{align}
	where $Z_{ij}=z_i-z_{j-1}$ if $1\leq j\leq i\leq 2m$ and $0$ else, and $V_i=v_i-v_0$ for $i,j=1,\dots, 2m$. As it holds true that
	\begin{align}
		\label{zx}
	z_{2j}=x_{2pj} = \frac{pj}{R}  ~ ~ \text{ for } ~ ~ 0\leq j\leq m ~ ~ \text{and} ~ ~ z_{2j-1}=x_{2pj-1} = \frac{pj}{R}-\delta ~ ~ \text{for} ~ ~ 1\leq j\leq m
	\end{align}
	the solution of \eqref{lin_sys} is given by \[w_0 = (z_1-z_0)^{-1}(v_1-v_0)=(\frac{p}{R}-\delta)^{-1}(v_1-v_0)\] and the recursion
	\begin{align}
		\label{rec:w_equi}
	w_i = (z_{i+1}-z_i)^{-1}(v_{i+1}-v_0-\sum_{j=0}^{i-1}w_j(z_{i+1}-z_j)).
\end{align}
 Noting that
	\[
		z_{2j+1}-z_{2j}=\frac{p}{R}-\delta=\frac{p(c+1)-1}{R(c+1)}\quad\text{and}\quad z_{2j}-z_{2j-1}=\delta
	\]
	under the assumption that $\delta = \frac{1}{(c+1)R}$ for some $c\in \mathbb{N}$ (yielding $c\delta = \frac{1}{R}-\delta$) the system may be extended to an equidistant system with linearly interpolated target data such that the considerations in the proof of Lemma \ref{lemm:equi_1} are applicable yielding 
	\begin{align}
		\label{w2_param_inequi}
		\notag w_{2j+1}&=(c+1)R(v_{2j+2}-(1+\frac{1}{p(c+1)-1})v_{2j+1}+\frac{1}{p(c+1)-1}v_{2j}),\\
		w_{2j}&=(c+1)R(\frac{1}{(c+1)p-1}v_{2j+1}-(1+\frac{1}{(c+1)p-1})v_{2j}+v_{2j-1}).
	\end{align}
	Indeed, it is straightforward to verify that the explicit representations in \eqref{w2_param_inequi} resolve the recursion in \eqref{rec:w_equi} and hence, solve \eqref{lin_sys}. Consequently, we derive that
		\begin{align}
			\label{lemm2.1_shen_complexity}
			\|\theta_\phi\|_\infty \lesssim\max(\vert x_0\vert, \vert x_{m(n+1)-1}\vert, \vert y_0\vert, cR \max_{0\leq i\leq 2\tilde{R}-1}\vert y_{i+1}-y_{i}\vert).
		\end{align}
		\pagebreak[2]
		
		This follows as $w_{2j+1}$ and $w_{2j}$ for $j=0, \dots, m-1$ given in \eqref{w2_param_inequi} are bounded by 
		\[
		(c+1)R(\frac{1}{p(c+1)-1}\max_{0\leq j\leq m-1}\vert y_{j(n+1)+n}-y_{j(n+1)}\vert+\max_{0\leq i\leq 2\tilde{R}-1}\vert y_{i+1}-y_i\vert)\leq2(c+1)RY.
		\]
		For the last inequality note that $\max_j\vert y_{j(n+1)+n}-y_{j(n+1)}\vert\leq n\max_i\vert y_{i+1}-y_i\vert$ and\[\frac{n}{p(c+1)-1}= \frac{2n}{(n+1)(c+1)-2}\leq \frac{2n}{2n}=1.\]
		Thus, using \eqref{lemm2.1_shen_complexity} together with the definition of $X, Y$ in \eqref{def:xy2} we derive that \[\|\theta_\phi\|_\infty \lesssim\max(X,cRY).\qedhere\]\end{proof}
	\paragraph{Complexity estimation in \cite[Lemma 5.4]{lu_main}:} 
	The result in \cite[Lemma 5.4]{lu_main} is the same as \cite[Lemma 2.2]{shen_external1}. It shows that the set of continuous piecewise linear functions with $m(n+1)$ pieces mapping an interval (with increasing breakpoints $x_i$ and corresponding values $y_i\geq 0$ for $0\leq i\leq m(n+1)$) to $\mathbb{R}$ is expressible by a two hidden layer ReLU network $\phi$ whose layers have widths $2m$ and $2n+1$, respectively. The underlying neural network $\phi$ may be realized in the form \[\phi(x)=W_3\relu(W_2\relu(W_1 x+b_1)+b_2)+b_3\] where $W_1\in\mathbb{R}^{2m\times 1}, W_2 \in \mathbb{R}^{(2n+1)\times2m}, W_3\in\mathbb{R}^{1\times (2n+1)}$, $b_1\in \mathbb{R}^{2m}, b_2\in\mathbb{R}^{2n+1}, b_3\in \mathbb{R}$ fulfilling $\phi(x_i)=y_i$ for $0\leq i\leq m(n+1)$.

Our result on the growth of parameters of this network is as follows.
	\begin{lemma}
		\label{lemm:inequi_2}
		In the situation of Lemma \ref{lemm:inequi_1} with $y_i\geq 0$ there exists a two hidden layer ReLU network $\phi$ whose layers have widths $2m$ and $2n+1$, respectively, fulfilling $\phi(x_i)=y_i$ for $0\leq i\leq m(n+1)$. Further, with $X$ and $Y$ defined as in \eqref{def:xy2}, it holds
		\[
			\|\theta_\phi\|_\infty \lesssim\max(X, c^3nRY).
		\]
	\end{lemma}
Inspecting the proof of \cite[Lemma 2.2]{shen_external1} one sees that the asymptotics of the parameter supremum of the parameters of $\phi$ is governed by the parameters $W_2$ and $b_2$. The reason is that $W_1, W_3, b_1, b_3$ can be chosen as $W_1 = (1,\dots,1)^T\in\mathbb{R}^{2m\times 1}$, 
\begin{align*}
 b_1 & = (-x_0, -x_n, -x_{n+1}, -x_{2n+1}, \dots, -x_{m(n+1)-1})\in\mathbb{R}^{2m}, \\
 W_3 & = (1,1,-1,1,-1,\dots, 1,-1)\in \mathbb{R}^{1\times(2n+1)},
 \end{align*}
$b_3=0$ and contribute to the asymptotics of $\|\theta_{\phi}\|_\infty$ only in terms of $X$ as in \eqref{def:xy}. Hence, to prove Lemma \ref{lemm:inequi_2}, we need to consider the construction of the parameters $W_2$ and $b_2$ in detail. These are given in terms of the parameters of inductively constructed single hidden layer ReLU-networks $g_k^+, g_k^-$ for $1\leq k\leq n$ and $g_0$ via

\begin{align*}
	 	g_0(x)&=(W_2)_1\relu(W_1x+b_1)+(b_2)_1,\\
	 	g_k^+(x) &= (W_2)_{2k}\relu(W_1 x+b_1)+(b_2)_{2k},\\
	 g_k^-(x) &= (W_2)_{2k+1}\relu(W_1 x+b_1)+(b_2)_{2k+1}
\end{align*}
	 for $1\leq k\leq n$, where $(W_2)_j$ denotes the $j$-th row of $W_2$.
	 
	 The single hidden layer ReLU networks are constructed to be linear on the intervals $[x_{j(n+1)}, x_{j(n+1)+n}]$ for $0\leq j\leq m-1$ and $[x_{j(n+1)-1}, x_{j(n+1)}]$ for $1\leq j\leq m$ such that, in the breakpoints $x_{j(n+1)}$ for $0\leq j\leq m$ and $x_{j(n+1)-1}$ for $1\leq j\leq m$, certain, in the following described, interpolation conditions are met.

Let $f_0$ be the piecewise linear continuous function fitting the data $y_i$ at $x_i$ for $0\leq i\leq m(n+1)$, which is linear on each of the subintervals $[x_{i-1},x_i]$ for $1\leq i\leq m(n+1)$.
The network $g_0$ is constructed such that $$g_0(x_{j(n+1)})=y_{j(n+1)}\quad \text{for}\quad0\leq j\leq m$$ $$\text{and}\quad g_0(x_{j(n+1)+n})=y_{j(n+1)+n}\quad \text{for} \quad 0\leq j\leq m-1.$$ Given $f_0$ and $g_0$ the network $f_1$ is defined by $f_1 := f_0-g_0$.
Following \cite[Lemma 2.2]{shen_external1} the networks $g_k^+$ and $g_k^-$ are constructed inductively for $1\leq k\leq n$ as follows.

 	Assuming $f_{k}$ to be given, the networks $g_k^+$, $g_k^-$ and $f_{k+1}$  are determined by the following conditions.
	For each $0\leq j\leq m$ if $f_k(x_{j(n+1)+k})\geq 0$ then $g_k^+, g_k^-$ must attain values at $x_{j(n+1)}, x_{j(n+1)+n}$ such that $$g_k^+(x_{j(n+1)+k})=f_k(x_{j(n+1)+k})\quad \text{and} \quad g_k^+(x_{j(n+1)+k-1})=0$$ and $g_k^-\equiv 0$ on $[x_{j(n+1)}, x_{j(n+1)+n}]$. Otherwise it holds true that $f_k(x_{j(n+1)+k})< 0$. Then $g_k^+, g_k^-$ must attain values at $x_{j(n+1)}, x_{j(n+1)+n}$ such that $$g_k^-(x_{j(n+1)+k})=-f_k(x_{j(n+1)+k})\quad \text{and} \quad g_k^-(x_{j(n+1)+k-1})=0$$ and $g_k^+\equiv 0$ on $[x_{j(n+1)}, x_{j(n+1)+n}]$. Also, $g_k^+$ and $g_k^-$ must be linear on the interval $[x_{j(n+1)}, x_{j(n+1)+n}]$.
	Finally, the function $f_{k+1}$ is defined by
	\[
		f_{k+1} := f_k-\relu(g_k^+)+\relu(g_k^-).
	\]
	
	We proceed by considering the following auxiliary results which provide a more explicit form of $f_k$ and will be useful later for estimating $$f_k(x_{j(n+1)+k})\quad \text{for}\quad 0\leq k,l\leq n.$$ Recall that these values essentially describe the data which is interpolated in the construction of the $g_k^+$ and $g_k^-$, hence they determine $W_2$ and $b_2$.
	
	Let us introduce the shorthand notation $$f_{k,l}^j= f_k(x_{j(n+1)+l})\quad \text{for}\quad 0\leq k,l\leq n, 0\leq j\leq m.$$ Then the $f_{k,l}^j$ fulfill the following recursion.
	\begin{lemma}
		\label{lemma:f_kl_one_step}
		In the situation of Lemma \ref{lemm:inequi_1} let $2\leq k\leq l\leq n$ and $\alpha = R\delta$. If
		\begin{itemize}
			\item $k$ is even and $l$ is even then $f_{k,l}^j=f_{k-1,l}^j+\frac{k-l-2}{2(1-\alpha)}f_{k-1,k-1}^j$.
			\item $k$ is even and $l$ is odd then
			$f_{k,l}^j=f_{k-1,l}^j+\frac{k-l-3+2\alpha}{2(1-\alpha)}f_{k-1,k-1}^j$.
			\item $k$ is odd and $l$ is even then
			$f_{k,l}^j=f_{k-1,l}^j+\frac{k-l-1-2\alpha}{2\alpha}f_{k-1,k-1}^j$.
			\item $k$ is odd and $l$ is odd then
			$f_{k,l}^j=f_{k-1,l}^j+\frac{k-l-2}{2\alpha}f_{k-1,k-1}^j$.
		\end{itemize}
	\end{lemma}
	\begin{proof}
		It holds true by the construction in \cite[Lemma 2.2]{shen_external1} that
		\begin{align}
			\label{f_kl_construction}
			f_{k,l}^j = f_{k-1}(x_{j(n+1)+l})-\relu(g_{k-1}^+(x_{j(n+1)+l}))+\relu(g_{k-1}^-(x_{j(n+1)+l})).
		\end{align}
		We determine $\relu(g_{k-1}^+(x_{j(n+1)+l}))$ and $\relu(g_{k-1}^-(x_{j(n+1)+l}))$ in terms of $f_{k-1}$ based on the setup of Lemma \ref{lemm:inequi_1}. For that, we recall that for $k\geq 1$ if $f_k(x_{j(n+1)+k})\geq 0$ then $g_k^+, g_k^-$ must attain values at $x_{j(n+1)}, x_{j(n+1)+n}$ such that $$g_k^+(x_{j(n+1)+k})=f_k(x_{j(n+1)+k})\quad \text{and} \quad g_k^+(x_{j(n+1)+k-1})=0$$ and $g_k^-\equiv 0$ on $[x_{j(n+1)}, x_{j(n+1)+n}]$. Note that  $g_k^+$ and $g_k^-$ are linear on the interval $[x_{j(n+1)}, x_{j(n+1)+n}]$.  In case $k$ is even, as then
		\[
		x_{j(n+1)+k}=x_{2(pj+k/2)}=\frac{pj+k/2}{R}=\frac{2pj+k}{2R}=\frac{(n+1)j+k}{2R}
		\]
		due to \eqref{zx}, the functions $g_k^+, g_k^-$ must fulfill
		\begin{itemize}
			\item $g_k^+(x_{j(n+1)})=(1-\frac{k}{2R\delta})f_k(x_{j(n+1)+k})$
			\item $g_k^+(x_{j(n+1)+n})=(\frac{2p-k}{2R\delta})f_k(x_{j(n+1)+k})$
			\item $g_k^-(x_{j(n+1)})=g_k^-(x_{j(n+1)+n})=0$.
		\end{itemize}
		In case $k$ is odd, as then
		\[
		x_{j(n+1)+k}=x_{2(pj+(k+1)/2)-1}=\frac{pj+(k+1)/2}{R}-\delta=\frac{(n+1)j+k+1}{2R}-\delta
		\]
		again due to \eqref{zx}, the functions $g_k^+, g_k^-$ must fulfill
		\begin{itemize}
			\item $g_k^+(x_{j(n+1)})=-\frac{k-1}{2(1-R\delta)}f_k(x_{j(n+1)+k})$
			\item $g_k^+(x_{j(n+1)+n})=(\frac{n-k+2(1-R\delta)}{2(1-R\delta)})f_k(x_{j(n+1)+k})$
			\item $g_k^-(x_{j(n+1)})=g_k^-(x_{j(n+1)+n})=0$.
		\end{itemize}
		Note that then independently of the parity of $k$ it holds $g_k^+(x_{j(n+1)+k})=f_k(x_{j(n+1)+k})$ and $g_k^+(x_{j(n+1)+k-1})=0$ as $g_k^+$ is linear on $[x_{j(n+1)}, x_{j(n+1)+n}].$\\
		
		If $f_k(x_{j(n+1)+k})< 0$ then $g_k^+, g_k^-$ must attain values at $x_{j(n+1)}, x_{j(n+1)+n}$ such that $$g_k^-(x_{j(n+1)+k})=-f_k(x_{j(n+1)+k})\quad \text{and} \quad g_k^-(x_{j(n+1)+k-1})=0$$ and $g_k^+\equiv 0$ on $[x_{j(n+1)}, x_{j(n+1)+n}]$. Note that  $g_k^+$ and $g_k^-$ are linear on the interval $[x_{j(n+1)}, x_{j(n+1)+n}]$. The concrete values are similar to those for the case  $f_k(x_{j(n+1)+k})\geq 0$ with exchanged roles of $g_k^+, g_k^-$ and opposite signs.\\
		
		Returning to identity \eqref{f_kl_construction} without loss of generality it suffices to consider the case $f_{k-1}(x_{j(n+1)+k-1})\geq 0$. The reason is that both $ g_{k-1}^+(x_{j(n+1)+l}), g_{k-1}^-(x_{j(n+1)+l})\geq0$ as $l\geq k$, the terms $g_{k-1}^+, g_{k-1}^-$ occur in \eqref{f_kl_construction} with different signs and $g_{k-1}^-(x_{j(n+1)+l})$ corresponds to $g_{k-1}^+(x_{j(n+1)+l})$ when applying the interpolation formula of $g_{k-1}^+$ to $-f_{k-1,k-1}^j$ instead of $f_{k-1,k-1}^j$.\\
		
		As a consequence, we may reduce \eqref{f_kl_construction} to 
		\begin{align}
			\label{f_kl_reduce}
			f_{k,l}^j=f_{k-1,l}^j-g_{k-1}^+(x_{j(n+1)+l}).
		\end{align}
		The next step is to write the term $g_{k-1}^+(x_{j(n+1)+l})$ as the linear interpolation between $g_{k-1}^+(x_{j(n+1)})$ and $g_{k-1}^+(x_{j(n+1)+n})$. By the previous considerations in case of even $k$ (odd $k-1$) it holds true that
		\[
		g_{k-1}^+(x_{j(n+1)+l})=-\frac{k-2}{2(1-\alpha)}f_{k-1,k-1}^j+\frac{x_{j(n+1)+l}-x_{j(n+1)}}{x_{j(n+1)+n}-x_{j(n+1)}}\frac{n-2\alpha+1}{2(1-\alpha)}f_{k-1,k-1}^j.
		\]
		Inserting even $l$ it yields
		\[
		g_{k-1}^+(x_{j(n+1)+l})=\frac{l-k+2}{2(1-\alpha)}f_{k-1,k-1}^j,
		\]
		whereas odd $l$ imply
		\[
		g_{k-1}^+(x_{j(n+1)+l})=\frac{3+l-k-2\alpha}{2(1-\alpha)}f_{k-1,k-1}^j.
		\]
		In case of odd $k$ (even $k-1$) it holds true that
		\[
		g_{k-1}^+(x_{j(n+1)+l})=(1-\frac{k-1}{2\alpha})f_{k-1,k-1}^j+\frac{x_{j(n+1)+l}-x_{j(n+1)}}{x_{j(n+1)+n}-x_{j(n+1)}}\frac{n-2\alpha+1}{2\alpha}f_{k-1,k-1}^j.
		\]
		Inserting even $l$ it yields
		\[
		g_{k-1}^+(x_{j(n+1)+l})=\frac{l-k+1+2\alpha}{2\alpha}f_{k-1,k-1}^j,
		\]
		whereas odd $l$ imply
		\[
		g_{k-1}^+(x_{j(n+1)+l})=\frac{l-k+2}{2\alpha}f_{k-1,k-1}^j.
		\]
		Together with equation \eqref{f_kl_reduce} this concludes the result.
	\end{proof}
	We proceed by showing how $f_{k,l}^j$ can be reduced to an expression containing only evaluations of the form $f_{1,i}^j$. 
	\begin{lemma}
		\label{lemma:fkl_resolved}
		In the situation of Lemma \ref{lemma:f_kl_one_step} let $2\leq k\leq l \leq n$. If
		\begin{itemize}
			\item $k$ is even and $l$ is even then $f_{k,l}^j=f_{1,l}^j-\frac{l-k+2}{2(1-\alpha)}f_{1,k-1}^j+\frac{l-k+2\alpha}{2(1-\alpha)}f_{1,k-2}^j$.
			\item $k$ is even and $l$ is odd then
			$f_{k,l}^j=f_{1,l}^j-\frac{l-k+3-2\alpha}{2(1-\alpha)}f_{1,k-1}^j+\frac{l-k+1}{2(1-\alpha)}f_{1,k-2}^j$.
			\item $k$ is odd and $l$ is even then
			$f_{k,l}^j=f_{1,l}^j-\frac{l-k+1+2\alpha}{2\alpha}f_{1,k-1}^j+\frac{l-k+1}{2\alpha}f_{1,k-2}^j$.
			\item $k$ is odd and $l$ is odd then
			$f_{k,l}^j=f_{1,l}^j-\frac{l-k+2}{2\alpha}f_{1,k-1}^j+\frac{l-k+2-2\alpha}{2\alpha}f_{1,k-2}^j$.
		\end{itemize}
	\end{lemma}
	
	\begin{proof}
		We prove the assertion via bilevel induction.\\
		
		\textit{Induction Start:} For $k=2$ and even $l$ it follows by Lemma \ref{lemma:f_kl_one_step} and $f_{1,0}^j=0$ that
		\begin{align*}
			f_{2,l}^j&=f_{1,l}^j-\frac{l}{2(1-\alpha)}f_{1,1}^j=f_{1,l}^j-\frac{l-2+2}{2(1-\alpha)}f_{1,1}^j+\frac{l-2+2\alpha}{2(1-\alpha)}f_{1,0}^j.
		\end{align*}
		Similarly for odd $l$ it follows
		\[
		f_{2,l}^j = f_{1,l}^j-\frac{l-2+3-2\alpha}{2(1-\alpha)}f_{1,1}^j+\frac{l-2+1}{2(1-\alpha)}f_{1,0}^j.
		\]
		
		For $k=3$ and even $l$ we derive
		\begin{align*}
			f_{3,l}^j&=f_{2,l}^j+\frac{2-l-2\alpha}{2\alpha}f_{2,2}^j\\
			&=(f_{1,l}^j-\frac{l}{2(1-\alpha)}f_{1,1}^j)+\frac{2-l-2\alpha}{2\alpha}(f_{1,2}^j-\frac{1}{1-\alpha}f_{1,1}^j)\\
			&=f_{1,l}^j-\frac{l-3+1+2\alpha}{2\alpha}f_{1,2}^j+\frac{l-3+1}{2\alpha}f_{1,1}^j.
		\end{align*}
		Similarly for odd $l$ it follows
		\begin{align*}
			f_{3,l}^j&=f_{2,l}^j+\frac{1-l}{2\alpha}f_{2,2}^j\\
			&=(f_{1,l}^j-\frac{l-2\alpha+1}{2(1-\alpha)}f_{1,1}^j)+\frac{1-l}{2\alpha}(f_{1,2}^j-\frac{1}{1-\alpha}f_{1,1}^j)\\
			&=f_{1,l}^j-\frac{l-3+2}{2\alpha}f_{1,2}^j+\frac{l-3+2-2\alpha}{2\alpha}f_{1,1}^j.
		\end{align*}
		
		\textit{Induction Step:} If $k+1$ is even and $l$ is even we derive by Lemma \ref{lemma:f_kl_one_step} and the induction hypothesis that
		\begin{align*}
			f_{k+1,l}^j&=f_{k,l}^j+\frac{k-l-1}{2(1-\alpha)}f_{k,k}^j\\
			&=(f_{1,l}^j-\frac{l-k+1+2\alpha}{2\alpha}f_{1,k-1}^j+\frac{l-k+1}{2\alpha}f_{1,k-2}^j)\\
			& ~ ~ ~ ~ +\frac{k-l-1}{2(1-\alpha)}(f_{1,k}^j-\frac{1}{\alpha}f_{1,k-1}^j+\frac{1-\alpha}{\alpha}f_{1,k-2}^j)\\
			&=f_{1,l}^j-\frac{l-k+1}{2(1-\alpha)}f_{1,k}^j+\frac{l-k+2\alpha-1}{2(1-\alpha)}f_{1,k-1}^j.
		\end{align*}
		If $k+1$ is even and $l$ is odd we derive
		\begin{align*}
			f_{k+1,l}^j&=f_{k,l}^j+\frac{k-l-2+2\alpha}{2(1-\alpha)}f_{k,k}^j\\
			&=(f_{1,l}^j-\frac{l-k+2}{2\alpha}f_{1,k-1}^j+\frac{l-k+2-2\alpha}{2\alpha}f_{1,k-2}^j)\\
			& ~ ~ ~ ~ +\frac{k-l-2+2\alpha}{2(1-\alpha)}(f_{1,k}^j-\frac{1}{\alpha}f_{1,k-1}^j+\frac{1-\alpha}{\alpha}f_{1,k-2}^j)\\
			&=f_{1,l}^j-\frac{l-k+2-2\alpha}{2(1-\alpha)}f_{1,k}^j+\frac{l-k}{2(1-\alpha)}f_{1,k-1}^j.
		\end{align*}
		If $k+1$ is odd and $l$ is even we derive
		\begin{align*}
			f_{k+1,l}^j&=f_{k,l}^j+\frac{k-l-2\alpha}{2\alpha}f_{k,k}^j\\
			&=(f_{1,l}^j-\frac{l-k+2}{2(1-\alpha)}f_{1,k-1}^j+\frac{l-k+2\alpha}{2(1-\alpha)}f_{1,k-2}^j)\\
			& ~ ~ ~ ~ +\frac{k-l-2\alpha}{2\alpha}(f_{1,k}^j-\frac{1}{1-\alpha}f_{1,k-1}^j+\frac{\alpha}{1-\alpha}f_{1,k-2}^j)\\
			&=f_{1,l}^j-\frac{l-k+2\alpha}{2\alpha}f_{1,k}^j+\frac{l-k}{2\alpha}f_{1,k-1}^j.
		\end{align*}
		If $k+1$ is odd and $l$ is odd we obtain
		\begin{align*}
			f_{k+1,l}^j&=f_{k,l}^j+\frac{k-l-1}{2\alpha}f_{k,k}^j\\
			&=(f_{1,l}^j-\frac{l-k+3-2\alpha}{2(1-\alpha)}f_{1,k-1}^j+\frac{l-k+1\alpha}{2(1-\alpha)}f_{1,k-2}^j)\\
			& ~ ~ ~ ~ +\frac{k-l-1}{2\alpha}(f_{1,k}^j-\frac{1}{1-\alpha}f_{1,k-1}^j+\frac{\alpha}{1-\alpha}f_{1,k-2}^j)\\
			&=f_{1,l}^j-\frac{l-k+1}{2\alpha}f_{1,k}^j+\frac{l-k+1-2\alpha}{2\alpha}f_{1,k-1}^j.
		\end{align*}
		Thus, the assertion holds also for $k+1$ which finishes the proof.
	\end{proof}
	The case $k=l$ yields the following identities.
	\begin{corollary}
		\label{corr: f_kkj}
		In the situation of Lemma \ref{lemma:f_kl_one_step} for even $k$ it holds true that
		\[
		f_{k,k}^j = f_{1,k}^j-\frac{1}{1-\alpha}f_{1,k-1}^j+\frac{\alpha}{1-\alpha}f_{1,k-2}^j
		\]
		and for odd $k$ that
		\[
		f_{k,k}^j=f_{1,k}^j-\frac{1}{\alpha}f_{1,k-1}^j+\frac{1-\alpha}{\alpha}f_{1,k-2}^j.
		\]
	\end{corollary}
	
	Having the above auxiliary results, we are now ready to prove Lemma \ref{lemm:inequi_2}.\\
	
	\textit{Proof of Lemma \ref{lemm:inequi_2}:} It suffices to estimate $\Vert W_2\Vert_\infty$ and $\Vert b_2\Vert_\infty$. Recall that the rows $(W_2)_{2k}$ and entries $(b_2)_{2k}$ correspond to the interpolation problem of $g_k^+$ and that the rows $(W_2)_{2k+1}$ and entries $(b_2)_{2k+1}$ correspond to the interpolation problem of $g_k^-$ for $1\leq k\leq n$ specified in the proof of Lemma \ref{lemma:f_kl_one_step} in the setup of Lemma \ref{lemm:inequi_2}. 
	
	As a consequence, we may apply the result in Lemma \ref{lemm:inequi_1}. Here we have to be careful regarding the term $Y$. It is the absolute value of the first interpolated point plus the maximal distance between two consecutive targets for $g_k^+$ and $g_k^-$, respectively. In either case $f_{k,k}^j<0$ or $f_{k,k}^j\geq 0$ for $0\leq k,l\leq n$, $0\leq j\leq m$, it holds true that the distance between two consecutive targets (for both $g_k^+$ and $g_k^-$) corresponding to \eqref{lemm2.1_shen_complexity} is bounded by
	\[
	\frac{n+1-2\alpha}{2\min(\alpha,1-\alpha)}\max_j \vert f_{k,k}^j\vert.
	\]
	The absolute value of the first interpolated point is bounded for both $g_k^+$, $g_k^-$ by
	\begin{align}
		\label{first_interp}
		\frac{n}{2\min(\alpha,1-\alpha)}\vert f_{k,k}^0\vert.
	\end{align}
	As $0<\alpha = \frac{1}{c+1}\leq \frac{1}{2}$ we obtain the estimation
	\[
	\frac{n+1-2\alpha}{2\min(\alpha,1-\alpha)}+\frac{n}{2\min(\alpha,1-\alpha)}\leq (n+1)(c+1).
	\]
	By Lemma \ref{lemm:inequi_1} the entries $(b_2)_k$ of the parameter $b_2$ are bounded by the absolute value of the first interpolated point given in \eqref{first_interp}. Again, as a consequence of Lemma \ref{lemm:inequi_1}, we obtain that
	\begin{align}
		\label{asymp:wrows}
		\vert (W_2)_{2k}\vert_\infty, \vert (W_2)_{2k+1}\vert_\infty, \vert (b_2)_{2k}\vert_\infty, \vert (b_2)_{2k+1}\vert_\infty\lesssim c^2nR\max_j\vert f_{k,k}^j\vert
	\end{align}
	for $1\leq k\leq n$. The next step is to describe the asymptotical behavior of the rows of $W_2$ independently of the row-index. By \eqref{lemm2.1_shen_complexity} we have that
	\[
	\vert (W_2)_1\vert_\infty\leq cR\max_j(\vert y_{j(n+1)+n}-y_{j(n+1)}\vert, \vert y_{j(n+1)}-y_{j(n+1)-1}\vert)\lesssim cnRY.
	\] 
	As $(b_2)_1= y_0$ the same asymptotical upper bound applies to $\vert (b_2)_1\vert_\infty$ too. For $k=1$ it holds true that $f_{1,1}^j= y_{j(n+1)+1}-g_0(x_{j(n+1)+1})$. Using that 
	\begin{align*}
		g_0(x_{j(n+1)+1})&=y_{j(n+1)}+\frac{x_{j(n+1)+1}-x_{j(n+1)}}{x_{j(n+1)+n}-x_{j(n+1)}}(y_{j(n+1)+n}-y_{j(n+1)})\\
		&= y_{j(n+1)}+\frac{2(1-\alpha)}{n+1-2\alpha}(y_{j(n+1)+n}-y_{j(n+1)})
	\end{align*}
	we obtain 
	\[
	\vert (W_2)_{2}\vert_\infty, \vert (W_2)_{3}\vert_\infty,\vert (b_2)_{2}\vert_\infty, \vert (b_2)_{3}\vert_\infty\lesssim c^2nR\max_j\vert f_{1,1}^j\vert\lesssim c^2nRY.
	\]
	Next we show this asymptotical behavior for general $k\geq 2$.\\
	
	More concretely, we verify for $k\geq 2$ that $\max_j\vert f_{k,k}^j\vert \lesssim cY$ and thus, by \eqref{asymp:wrows}
	\[
	\vert(W_2)_{2k}\vert_\infty, \vert (W_2)_{2k+1}\vert_\infty,\vert (b_2)_{2k}\vert_\infty, \vert (b_2)_{2k+1}\vert_\infty\lesssim c^3nRY.
	\]
	Indeed as a consequence of Corollary \ref{corr: f_kkj} we obtain that
	\[
	\vert f_{k,k}^j\vert \leq \vert f_{1,k}^j-f_{1,k-1}^j\vert+\max(\frac{1-\alpha}{\alpha}, \frac{\alpha}{1-\alpha})\vert f_{1,k-1}^j-f_{1,k-2}^j\vert.
	\]
	The estimations $\max(\frac{1-\alpha}{\alpha}, \frac{\alpha}{1-\alpha})\leq c$ and 
	\begin{align}
		\label{diff_f_1k_estim}
		\notag\vert f_{1,i+1}^j-f_{1,i}^j\vert &\leq \vert y_{j(n+1)+i+1}-y_{j(n+1)+i}\vert+\vert g_0(x_{j(n+1)+i+1})-g_0(x_{j(n+1)+i})\vert\\
		\notag&\leq Y+\frac{x_{j(n+1)+i+1}-x_{j(n+1)+i}}{x_{j(n+1)+n}-x_{j(n+1)}}\vert y_{j(n+1)+n}-y_{j(n+1)}\vert\\
		\notag&\leq Y(1+n(\frac{1}{R}-\delta)(\frac{n+1}{2R}-\delta)^{-1})\\
		\notag&=Y(1+\frac{2n(1-\alpha)}{n+1-2\alpha})\\
		&\leq 3Y,
	\end{align}
	infer that
	\[
	\max_j\vert f_{k,k}^j\vert\leq 3(c+1)Y.
	\]
	Recalling that for $k\geq 2$ 
	\[
	\vert (W_2)_{2k}\vert_\infty, \vert (W_2)_{2k+1}\vert_\infty, \vert (b_2)_{2k}\vert_\infty, \vert (b_2)_{2k+1}\vert_\infty\lesssim c^2nR\max_j\vert f_{k,k}^j\vert
	\]
	we conclude that
	\[
	\vert (W_2)_{2k}\vert_\infty, \vert (W_2)_{2k+1}\vert_\infty,\vert (b_2)_{2k}\vert_\infty, \vert (b_2)_{2k+1}\vert_\infty\lesssim c^3nRY
	\]
	and finally, the growth of $\|\theta_\phi\|_\infty $ can be estimated by 
	\[
	\|\theta_\phi\|_\infty \lesssim\max(X,c^3nRY).\vspace*{-0.8cm}
	\]
	\qed\\

	Similar arguments yield the growth of parameters of the network in \cite[Lemma 2.2]{shen_external1} for an equidistant grid.
	\begin{corollary}
		\label{corr:equidist}
		In the situation of Lemma \ref{lemm:equi_1} with $y_i\geq 0$ assume that there exist $m, n \in \mathbb{N}$ with $\tilde{R}=m(n+1)$. Then there exists a two hidden layer ReLU network $\phi$ whose layers have widths $2m$ and $2n+1$, respectively, fulfilling $\phi(x_i)=y_i$ for $0\leq i\leq m(n+1)$. Furthermore, it holds, with $X$ and $Y$ defined as in \eqref{def:xy}, that
		\[
			\|\theta_\phi\|_\infty\lesssim\max(X,nRY).
		\]
	\end{corollary}
	\begin{proof}
	The arguments are essentially the same as for the proof of Lemma \ref{lemm:inequi_2}. For that reason we omit an explicit proof since the more complicated inequidistant case is dealt with in Lemma \ref{lemm:inequi_2}. Nevertheless, we give the main steps.
	
	The main observation is that an inequidistant grid of the form given in Lemma \ref{lemm:inequi_1} transforms into an equidistant grid of the form given in Lemma \ref{lemm:equi_1} for $c=1$, though scaled. As a consequence, it holds true that $\alpha = R\delta = R/(2R) = 1/2$. Following the proofs of Lemma \ref{lemma:f_kl_one_step}, Lemma \ref{lemma:fkl_resolved} and Corollary \ref{corr: f_kkj}, one obtains that
	\[
		f_{0,0}^j = y_{j(n+1)},\qquad f_{1,1}^j = (y_{j(n+1)+1}- y_{j(n+1)})-\frac{1}{n}(y_{j(n+1)+n}-y_{j(n+1)})
	\]
	and for $k\geq 2$
	\[
		f_{k,k}^j = f_{1,k}^j-2f_{1,k-1}^j+f_{1,k-2}^j
	\]
	for $0\leq j\leq m$. By \eqref{diff_f_1k_estim} it follows that $\max_j\vert f_{k,k}^j\vert \leq 6Y$. Thus, as by Lemma \ref{lemm:equi_1}
	\[
	\vert (W_2)_{2k}\vert_\infty, \vert (W_2)_{2k+1}\vert_\infty,\vert (b_2)_{2k}\vert_\infty, \vert (b_2)_{2k+1}\vert_\infty\lesssim nR\max_j\vert f_{k,k}^j\vert
	\]
	we conclude that
	\[
	\vert (W_2)_{2k}\vert_\infty, \vert (W_2)_{2k+1}\vert_\infty,\vert (b_2)_{2k}\vert_\infty, \vert (b_2)_{2k+1}\vert_\infty\lesssim nRY
	\]
	and finally, the growth of the parameters of $\phi$ is given by 
	\[
	\|\theta_\phi\|_\infty \lesssim\max(X,nRY).\qedhere
	\]
	\end{proof}

	We conclude that for the equidistant grid the growth of parameters amounts to 
	\[
		\|\theta_\phi\|_\infty \lesssim nRY
	\]
	and for the special inequidistant grid in Lemma \ref{lemm:inequi_2} to
	\[
		\|\theta_\phi\|_\infty \lesssim c^3nRY.
	\]
	
	\paragraph{Complexity estimation in \cite[Lemma 5.5]{lu_main}:}
	The result in \cite[Lemma 5.5]{lu_main} is the same as \cite[Lemma 3.4]{shen_external2}. It shows that a two hidden layer ReLU network with $d-$dimensional input and layers of width $N$ and $NL$, respectively, of the form
	\[
		\tilde{\phi}(x)=W_3\relu(W_2\relu(W_1x+b_1)+b_2)+b_3,
	\]
	can be expressed by a ReLU network $\phi$ with $d-$dimensional input, width $\mathcal{W}(\phi)=\mathcal{O}(N)$ and depth $\mathcal{D}(\phi)=\mathcal{O}(L)$. More concretely, for
	\[
		\mathbf{g}=\relu(W_1x+b_1)\quad\text{and}\quad \mathbf{h}=\relu(W_2 \mathbf{g}+b_2)
	\]
	defining the outputs of the intermediate layers of $\tilde{\phi}$, the $L+1$ intermediate layers of the proposed deep ReLU network $\phi$ of width $\mathcal{O}(N)$ and depth $\mathcal{O}(L)$ consist of expressions of the form $\mathbf{g}, \mathbf{h}_i$, $\relu(s_i)$ and $\relu(-s_i)$ for $1\leq i\leq L$. Dividing $W_2\in \mathbb{R}^{NL\times N}$ and $b_2\in \mathbb{R}^{NL}$ evenly into $L$ parts $W_{2,i}\in \mathbb{R}^{N\times N}$ and $b_{2,i}\in \mathbb{R}^{N}$, respectively, the $\mathbf{h}_i$ are obtained by
	\[
		\mathbf{h}_i=\relu(W_{2,i}\mathbf{g}+b_{2,i})\qquad\text{for } 1\leq i\leq L.
	\]
	Similarly dividing $W_3\in \mathbb{R}^{1\times NL}$ evenly into parts of length $N$ denoted by $W_{3,i}$, the $s_i$ are recursively defined by
	\[
		s_0:=0\qquad\text{and}\qquad s_i:=s_{i-1}+W_{3,i} \mathbf{h}_i\qquad\text{for } 1\leq i\leq L.
	\]
	We argue that the necessary parameters for constructing these expressions, determining the intermediate layers of $\phi$, consist of entries of the parameters of $\tilde{\phi}$ given by $W_l, b_l$ for $l=1,2,3$.
	\begin{itemize}[leftmargin=0.5cm]
		\item The parameters resulting in the output of the first intermediate layer of $\phi$ given by $\mathbf{g}$ are $W_1$ and $b_1$ as $\mathbf{g} = \relu(W_1x+b_1)$.
		\item The output of the second layer is given by $\mathbf{h}_1$ and $\mathbf{g}$ and obtained by $W_{2,1}$ and $b_{2,1}$ due to the definition of $\mathbf{h}_1$ and constants $\pm 1$ as $\mathbf{g} = \relu(\mathbf{g})-\relu(-\mathbf{g})$.
		\item For $3\leq i\leq L+1$ the output of the $i-$th intermediate layer of $\phi$ is formed by $\relu(s_{i-2}), \relu(-s_{i-2}), \mathbf{h}_{i-1}$ and $\mathbf{g}$. As a consequence of the previous considerations and recursion formula of the $s_i$, the parameters necessary to obtain the $i$-th layer output by the previous layer are $W_{2,i-1}, b_{2,i-1}$, $W_{3,i-2}$ and constants given by $\pm 1$ (to map $\mathbf{g}$ to $\mathbf{g}$ and $\relu(s_i)$ with $\relu(-s_i)$ to $s_i$).
		\item Finally, the necessary parameters mapping the penultimate layer of $\phi$ to the output layer consisting of $\phi(x)=W_3\mathbf{h}+b_3 = s_L+b_3$ are $b_3$, $W_{3,L}$ and again constants given by $\pm 1$.
	\end{itemize}
	Thus, the asymptotical behavior of the parameters in \cite[Lemma 3.4]{shen_external2} is given by
	\[
		\|\theta_\phi\|_\infty \lesssim\|\theta_{\tilde{\phi}}\|_\infty \lesssim\max((\vert b_i\vert_\infty)_{i=1,2,3}, (\vert W_i\vert_\infty)_{i=1,2,3}).
	\]

	\paragraph{Complexity estimation in \cite[Proposition 4.3]{lu_main}:} The result in \cite[Proposition 4.3]{lu_main} shows that there exists a ReLU FNN $\phi$ with width $\mathcal{W}(\phi)=\mathcal{O}(N^{1/d})$ and depth $\mathcal{D}(\phi)=\mathcal{O}(L)$ realizing the step function fulfilling
	\[
		\phi(x)=k\quad\text{if } x\in [\frac{k}{K},\frac{k+1}{K}-\delta1_{\left\{k\leq K-2\right\}}]\quad\text{for } 0\leq k\leq K-1
	\]
	with $K=\lfloor N^{1/d}\rfloor^2\lfloor L^{2/d}\rfloor$. The cases $d=1$ and $d\geq 2$ are considered separately.\\
	
	 \textbf{Case} $\boldsymbol{d=1}$\textbf{:} The network $\phi$ is given by $\phi(x)=\phi_1(x)L+\phi_2(x-M^{-1}\phi_1(x))$ where $M=N^2L$ and $\phi_1$, $\phi_2$ are defined as follows. The $\phi_1$ is a ReLU network with
	\[
		\phi_1(x)=m\quad\text{if } x\in[\frac{m}{M},\frac{m+1}{M}-\delta1_{\left\{m\leq M-2\right\}}]\quad\text{for } 0\leq m\leq M-1.
	\]
	The $\phi_2$ is a ReLU network with
	\[
		\phi_2(x)=l\quad\text{if } x\in [\frac{l}{ML},\frac{l+1}{ML}-\delta1_{\left\{l\leq L-2\right\}}]\quad\text{for } 0\leq l\leq L-2.
	\] In view of the considerations on \cite[Lemma 2.2]{shen_external1} we may apply Lemma \ref{lemm:inequi_2} to $\phi_1$ with $R=N^2L, m=N, n = 2NL-1$, $Y=1, X=2$, $\delta = \frac{1}{(c+1)R}$ and thus, obtain
	\[
	 \|\theta_{\phi_1}\|_\infty  \lesssim c^3N^3L^2.
	\]
	Similarly for $\phi_2$ the result in Lemma \ref{lemm:inequi_2} is applicable with $R = N^2L^2$, $m=1$, $n=2L-1$, $Y=1, X=2$, $\delta = \frac{1}{(c+1)R}$ and hence,
	\[
	\|\theta_{\phi_2}\|_\infty \lesssim c^3N^2L^3.
	\]
	As we have discussed in the paragraph above, the application of \cite[Lemma 3.4]{shen_external2} to $\phi_1,\phi_2$ resulting in modified networks with $\mathcal{W}(\phi_1)=\mathcal{O}(N), \mathcal{D}(\phi_1)=\mathcal{O}(L), \mathcal{W}(\phi_2)=\mathcal{O}(1)$ and $\mathcal{D}(\phi_2)=\mathcal{O}(L)$, does not increase the complexity of the parameters such that by definition of $\phi$ (see also \cite[Figure 13]{lu_main}) it holds true that
	\[
		\|\theta_\phi\|_\infty \lesssim c^3N^2L^2(N+L)
	\]
	with $\mathcal{W}(\phi)=\mathcal{O}(N)$ and $\mathcal{D}(\phi)=\mathcal{O}(L)$.\\
	
	\textbf{Case} $\boldsymbol{d\geq 2}$\textbf{:} The result in Lemma \ref{lemm:inequi_2} is applicable to $\phi$ with
	\[
		R= \lfloor N^{1/d}\rfloor^2\lfloor L^{2/d}\rfloor,\quad m=\lfloor N^{1/d}\rfloor,\quad n= 2\lfloor N^{1/d}\rfloor\lfloor L^{2/d}\rfloor-1,\quad Y=1,\quad X=2
	\]
	and as a consequence,
	\[
	\|\theta_\phi\|_\infty \lesssim c^3N^{3/d}L^{4/d}.
	\]
	Again the application of \cite[Lemma 3.4]{shen_external2} to $\phi$ resulting in a modified network with $\mathcal{W}(\phi)=\mathcal{O}(N^{1/d})$ and $\mathcal{D}(\phi)=\mathcal{O}(L)$ fulfilling the properties above does not increase the complexity of the parameters. Note that the asymptotical bound for $d\geq 2$ applies also in the case $d=1$ but is worse.

	\paragraph{Complexity estimation in \cite[Lemma 5.6]{lu_main}:} The result in \cite[Lemma 5.6]{lu_main} is the same as \cite[Lemma 3.6]{shen_external2} and corresponds to a certain bit-extraction technique. More concretely, it shows that for given $\theta_{m,l}\in\left\{0,1\right\}$ for $0\leq m\leq M-1,0\leq l\leq L-1$ with $M=N^2L$ there exists a ReLU network $\phi$ with $\mathcal{W}(\phi)=\mathcal{O}(N)$ and $\mathcal{D}(\phi)=\mathcal{O}(L)$ such that
	\[
		\phi(m,l)=\sum_{j=0}^l\theta_{m,j}\qquad\text{for } 0\leq m\leq M-1, 0\leq l\leq L-1.
	\] The $\phi$ is given by $\phi(m,l)=\phi_2(\phi_1(m),l+1)$ where $\phi_1$ is a ReLU network mapping $m$ to the unique real number $y_m$ which has the coefficients $(\theta_{m,l})_{0\leq l\leq L-1}$ in its binary representation and $\phi_2$ is a ReLU network mapping $x,l$ to the sum of the first $l$ coefficients of the binary representation of $x$.\\
	In view of the considerations on \cite[Lemma 2.2]{shen_external1} applying Lemma \ref{lemm:inequi_2} to $\phi_1$ with $R=1, m=N, n= NL-1, Y\leq1, X=N^2L$ yields $\|\theta_{\phi_1}\|_\infty\lesssim N^2L$. The application of \cite[Lemma 3.4]{shen_external2} to $\phi_1$ resulting in a modified network with $\mathcal{W}(\phi_1)=\mathcal{O}(N)$ and $\mathcal{D}(\phi_1)=\mathcal{O}(L)$ does not increase the complexity of the parameters.\\
	To analyze the parameters of the network $\phi_2$ with $\mathcal{W}(\phi_2)=\mathcal{O}(1)$ and $\mathcal{D}(\phi_2)=\mathcal{O}(L)$ we consider the transformations between the layers of $\phi_2$. For given $(\theta_l)_{1\leq l\leq L}$ and $\xi_j$ the real number attaining binary coefficients $(\theta_l)_{j\leq l\leq L}$ the recursion formulas
	\begin{align}
		\label{thetarecursion}
		\theta_j&=\relu(2^L(\xi_j-1/2)+1)-\relu(2^L(\xi_j-1/2))\\
		\label{xirecursion}
		\xi_{j+1}&= 2\xi_j-\relu(2^L(\xi_j-1/2)+1)+\relu(2^L(\xi_j-1/2))
	\end{align}
	are shown to hold true in \cite[Lemma 3.5]{shen_external2}. Furthermore, for 
	\begin{align}
	\label{zrecursion}
	z_{l,j}=\relu(\theta_j+\relu(l-j+1)-\sigma(l-j)-1)
	\end{align}
	it is argued that\vspace*{-0.1cm}
	\[
		\sum_{j=1}^l\theta_j=\sum_{j=1}^L z_{l,j}.
	\]
	The formulas \eqref{thetarecursion}-\eqref{zrecursion} are employed in the intermediate layers of $\phi_2$ to generate $\xi_j$, $\theta_j$, $z_{l,j}$ recursively for $1\leq l\leq L$ and finally output $$\sum_{j=1}^l\theta_j=\sum_{j=1}^L z_{l,j}$$ in the last layer. As a consequence, it holds true that $\|\theta_{\phi_2}\|_\infty\lesssim2^L$ due to the multiplication by $2^L$ occurring in the formulas \eqref{thetarecursion} and \eqref{xirecursion}.

	The fact that the complexity behaves asymptotically exponential in the depth $L$ is undesirable. One can circumvent this by adapting the constructive proof of \cite[Lemma 3.5]{shen_external2} as follows. One can introduce a network $\phi_3$ of width equal to $2$ and $L$ layers with $\|\theta_{\phi_3}\|_\infty=2$ realizing the multiplication by $2^L$. Then the network $\phi_2$ can be modified by applying $\phi_3$ to $\xi_j-1/2$ (compare to formulas \eqref{thetarecursion} and \eqref{xirecursion}) right before those intermediate layers of $\phi_2$ where $\theta_j$ and $\xi_{j+1}$ are generated for $1\leq j\leq L$. As a consequence in \cite[Lemma 3.5]{shen_external2} we obtain $\|\theta_{\phi_2}\|_\infty\lesssim L$ with modified $\phi_2$. Note that the depth of the modified network is no more linear in $L$ but quadratic and given by $L^2+L+1 = \mathcal{O}(L^2)$ and similarly also for the network in \cite[Lemma 3.6]{shen_external2} which, with the modification above, fulfills
	\[
		\|\theta_{\phi}\|_\infty\lesssim N^2L.
	\]
	
	\paragraph{Complexity estimation in \cite[Lemma 5.7]{lu_main}:} The result in \cite[Lemma 5.7]{lu_main} corresponds to a modified bit-extraction technique. More concretely, for $N,L\in \mathbb{N}$ and $\theta_i\in\left\{0,1\right\}$ for $0\leq i\leq N^2L^2-1$ the network $\phi$ realizes $\phi(i)=\theta_i$ for $0\leq i\leq N^2L^2-1$ with $\mathcal{W}(\phi)=\mathcal{O}(N)$ and $\mathcal{D}(\phi)=\mathcal{O}(L)$. The network $\phi$ fulfills
	\[
		\phi(i)=\phi_1(\psi(i),i-L\psi(i))-\phi_2(\psi(i),i-L\psi(i))
	\]
	where $\phi_1,\phi_2,\psi$ are defined as follows. Defining $a_{m,l}:=\theta_i$ if $i=mL+l$ for $0\leq m\leq N^2L-1, 0\leq l\leq L-1$, $b_{m,0}=0$ and $b_{m,l}=a_{m,l-1}$ for $0\leq m\leq N^2L-1,1\leq l\leq L-1$, the networks $\phi_1,\phi_2$ fulfill that
	\[
		\phi_1(m,l)=\sum_{j=0}^l a_{m,j}\qquad\text{and}\qquad \phi_2(m,l)=\sum_{j=0}^l b_{m,j}
	\]
	for $0\leq m\leq N^2L-1, 0\leq l\leq L-1$. The asymptotics of the parameters of $\phi_1,\phi_2$ are governed under the previously discussed modifications in \cite[Lemma 5.6]{lu_main} by the complexity $\|\theta_{\phi_i}\|_\infty\lesssim N^2L$ where $\mathcal{W}(\phi_i)=\mathcal{O}(N)$ and $\mathcal{D}(\phi_i)=\mathcal{O}(L^2)$ for $i=1,2$. Finally, the $\psi$ is a ReLU network with $\psi(x)=m$ if $x\in \left[mL,(m+1)L-1\right]$ for $0\leq m\leq M-1$. We may apply Lemma \ref{lemm:inequi_2} to $\psi$ with
	\[
		R=1/L,\quad m=N,\quad n=2NL-1,\quad Y=1,\quad X=N^2L^2,\quad \delta=1=\frac{1}{(c+1)R},
	\]
	i.e., $c=L-1$, that $\|\theta_{\psi}\|_\infty\lesssim NL^3$. Again the application of \cite[Lemma 3.4]{shen_external2} resulting in a modified network $\psi$ with $\mathcal{W}(\psi)=\mathcal{O}(N)$ and $\mathcal{D}(\psi)=\mathcal{O}(L)$ does not increase the complexity of the parameters. Thus, the growth of parameters of the realizing network in \cite[Lemma 5.7]{lu_main} is given by
	\[
		\|\theta_\phi\|_\infty \lesssim NL(N+L^2).
	\]
	
	\paragraph{Complexity estimation in \cite[Proposition 4.4]{lu_main}:} The result in \cite[Proposition 4.4]{lu_main} shows that fitting partial derivatives of order at most $q-1$ at the corners of the subcubes to which $\Psi$ projects to (see \cite[Proposition 4.4]{lu_main}) is realizable by a ReLU FNN. More concretely, for $N,L, s\in \mathbb{N}$ and $0\leq \xi_i\leq 1$ for $0\leq i\leq N^2L^2-1$ there exists a ReLU FNN $\phi$ with $\mathcal{W}(\phi)=\mathcal{O}(sN\log N)$ and $\mathcal{D}(\phi)=\mathcal{O}(L^2\log L)$ such that $\vert \phi(i)-\xi_i\vert\leq N^{-2s}L^{-2s}$ for $0\leq i\leq N^2L^2-1$ and $0\leq \phi\leq 1$. 
	
	For that, the ReLU FNN $\phi_j$ realizing $\phi_j(i)=\xi_{i,j}$ for $0\leq i\leq N^2L^2-1$ are introduced for $1\leq j\leq J:=\lceil 2s \log(NL+1)\rceil$ where $\xi_{i,j}\in\left\{0,1\right\}$ are such that the real number with binary coefficients given by $(\xi_{i,j})_{0\leq j\leq J}$ approximates $\xi_i$ with error bounded by $2^{-J}$. For $$\tilde{\phi}(x):=\sum_{j=1}^J 2^{-j}\phi_j(x)$$ it follows by the considerations in the previous paragraph that $\|\theta_{\tilde \phi}\|_\infty\lesssim NL(N+L^2)$. Finally, the network $\phi$ is defined by
	\[
		\phi(x)=\min(\relu(\tilde{\phi}(x)),1)
	\]
	which, as the minimum is expressible by ReLU FNN with parameters $\pm 1$ and constant architecture, fulfills	the growth of parameters
	\[
		\|\theta_\phi\|_\infty \lesssim NL(N+L^2).
	\]
	
	\paragraph{Complexity estimation in \cite[Lemma 5.1]{lu_main}:} The result in \cite[Lemma 5.1]{lu_main} shows that the function $x\mapsto x^2$ may be approximated with error $N^{-L}$ on the unit interval by a ReLU FNN $\phi$ with width $\mathcal{W}(\phi)=3N$ and depth $\mathcal{D}(\phi)=L$. The network $\phi$ is given by
	\[
		\phi(x)=x-\sum_{i=1}^{Lk}2^{-2i}T_i(x)
	\]
	with $k\in \mathbb{N}$ uniquely given such that $(k-1)2^{k-1}+1\leq N\leq k2^k$ and $T_i$ sawtooth functions fulfilling $T_i(l2^{-i})=1$ for odd $0\leq l\leq 2^i$ and $T_i(l2^{-i})=0$ for even $0\leq l\leq 2^i$. The complexity of $\phi$ is governed by the growth of parameters of the sawtooth function $T_k$. The reason is that $T_1, \dots, T_k$ which are generated in the first intermediate layer of $\phi$, are transformed to the higher order $T_i$ by application of $T_k$. By Lemma \ref{lemm:equi_1} as $T_k(l2^{-k})=1$ for odd $0\leq l\leq 2^k$ and $T_k(l2^{-k})=0$ for even $0\leq l\leq 2^k$  it follows that $\|\theta_{\phi}\|_\infty\lesssim 2^k$. By the choice of $k$ it holds $2^k\leq N$ and consequently the asymptotical behavior of the parameters of $\phi$ fulfills
	\[
		\|\theta_\phi\|_\infty \lesssim N.
	\]
	
	\paragraph{Complexity estimation in \cite[Lemma 5.2]{lu_main}:} The result in \cite[Lemma 5.2]{lu_main} shows that the function $(x,y)\mapsto xy$ may be approximated with error of order $N^{-L}$ on the unit square by a ReLU FNN $\phi$ with width $\mathcal{W}(\phi)=9N$ and depth $\mathcal{D}(\phi)=L$. For $\psi$ denoting the network providing the result of the previous paragraph, the network $\phi$ is defined by
	\[
		\phi(x,y)=2(\psi((x+y)/2)-\psi(x/2)-\psi(y/2)).
	\] As a consequence, we derive that the complexity of $\phi$ follows immediately by the considerations on \cite[Lemma 5.1]{lu_main} and fulfills \[\|\theta_\phi\|_\infty \lesssim N.\]
	
	\paragraph{Complexity estimation in \cite[Lemma 4.2]{lu_main}:} The result in \cite[Lemma 4.2]{lu_main} shows that the function $(x,y)\mapsto xy$ may be approximated with error of order $N^{-L}$ on a general square $[a,b]^2$ by a ReLU FNN $\phi$ with width $\mathcal{W}(\phi)=9N+1$ and depth $\mathcal{D}(\phi)=L$. For $\psi$ denoting the network providing the result of the previous paragraph, the network $\phi$ is defined by 
	\[
		\phi(x,y):=(b-a)^2\psi(\frac{x-a}{b-a}, \frac{y-a}{b-a})+a\relu(x+y+2\vert a\vert)-a^2-2a\vert a\vert.
	\]
	Thus, we obtain that the complexity of $\phi$ follows immediately by the considerations on \cite[Lemma 5.2]{lu_main} and fulfills \[\|\theta_\phi\|_\infty \lesssim N.\]
	
	\paragraph{Complexity estimation in \cite[Lemma 5.3]{lu_main}:} The result in \cite[Lemma 5.3]{lu_main} shows that multivariable functions of the form $$(x_1, \dots, x_k)\mapsto x_1x_2\dots x_k$$ on the $k-$unit cube may be approximated with error of order $N^{-7kL}$ by a ReLU FNN $\phi$ with width $\mathcal{W}(\phi)=\mathcal{O}(N)$ and depth $\mathcal{D}(\phi)=\mathcal{O}(L)$. For $\phi_1$ denoting the the network providing the result of the previous paragraph, the networks $\phi_i$ are recursively defined by
	\[
		\phi_{i+1}(x_1,\dots, x_{i+2}):=\phi_1(\phi_i(x_1,\dots, x_{i+1}), \relu(x_{i+2}))
	\]
	for $x_1, \dots, x_{i+2}\in \mathbb{R}$ and $\phi$ is defined by $\phi:=\phi_{k-1}$. Hence, the estimate on $\|\theta_{\phi}\|_\infty$ is a direct corollary of the previous paragraph as the network in \cite[Lemma 4.2]{lu_main} is self-composed $k-1$ times to obtain the network in \cite[Lemma 5.3]{lu_main}. As a consequence, we derive that \[\|\theta_\phi\|_\infty \lesssim N.\]
	
	\paragraph{Complexity estimation in \cite[Proposition 4.1]{lu_main}:} The result in \cite[Proposition 4.1]{lu_main} shows that multivariable polynomials $P(x)=x^\alpha$ of $d$ variables and degree $\tilde{k}:=\vert \alpha\vert_1\leq k$ can be approximated on the $d-$unit cube with error of order $N^{-7kL}$ by a ReLU FNN $\phi$ with width $\mathcal{W}(\phi)=\mathcal{O}(N+k)$ and depth $\mathcal{D}(\phi)=\mathcal{O}(k^2L)$. 
	
	 Denoting by $\psi$ the network providing the result of the previous paragraph and by $\mathcal{L}:\mathbb{R}^d\to \mathbb{R}^k$ the affine linear map which will be recalled from \cite[Lemma 5.3]{lu_main}, the network $\phi$ is defined by $\phi = \psi\circ\mathcal{L}$. Given $x\in \mathbb{R}^d$, $z\in \mathbb{R}^{\tilde{k}}$ is defined by
	\[
		z_l = x_j\quad\text{if } \sum_{i=1}^{j-1}\alpha_i< l\leq \sum_{i=1}^j\alpha_i\quad\text{for } 1\leq j\leq d,
	\]
	i.e., $z$ is the entrywise replication of $x$ with respect to $\alpha$. Then $x\in \mathbb{R}^d$ is mapped to $(z, 1,\dots, 1)^T\in \mathbb{R}^k$ by $\mathcal{L}$. The affine linear map $\mathcal{L}$ can be expressed by a ReLU FNN where each nonzero scalar parameter is equal to $1$. Thus, the estimate on $\|\theta_{\phi}\|_\infty$ is a corollary of the considerations of the previous paragraph and fulfills \[\|\theta_\phi\|_\infty \lesssim N.\]
	
	\paragraph{Complexity estimation in \cite[Theorem 2.2]{lu_main}:} The approximating neural network in \cite[Theorem 2.2]{lu_main} on $\Omega([0,1]^d, R,\delta)$ for $R = \lfloor N^{1/d}\rfloor^2\lfloor L^{2/d}\rfloor$ is given by
	\[
	\phi(x):=\sum_{\Vert\alpha\Vert_1\leq q-1}\varphi(\frac{1}{\alpha!}\phi_\alpha(\Psi(x)), P_\alpha(x-\Psi(x)))
	\]
	for $x\in \mathbb{R}^d$, where the role of the subnetworks $\Psi, P_\alpha, \phi_\alpha, \varphi$ is as follows:
	\begin{itemize}[leftmargin=0.5cm]
		\item The ReLU FNN $\Psi$ realizes projections of subcubes of $[0,1]^d$ to exactly one corner of the subcube based on one-dimensional step functions $\psi$ with $\Psi(x)=(\psi(x_1), \dots, \psi(x_d))^T/R$ for $x\in [0,1]^d$ (see considerations on \cite[Proposition 4.3]{lu_main} for construction of $\psi$). The realized analysis revealed that the growth of parameters of $\Psi$ fulfills
		\[
			\|\theta_\Psi\|_\infty \lesssim c^3N^{3/d}L^{4/d}.
		\]
		\item The ReLU FNN $P_\alpha$ achieves an approximation of multinomials of order at most $q-1$ (see considerations on \cite[Proposition 4.1]{lu_main}) and is shown in the complexity estimations above to fulfill
		\[
			\|\theta_{P_\alpha}\|_\infty \lesssim N.
		\]
		\item The ReLU FNN $\phi_\alpha$ achieves fitting partial derivatives of $f$ of order at most $q-1$ at the subcube corners to which $\Psi$ projects to (see considerations on \cite[Proposition 4.4]{lu_main}). For the growth of parameters we derive
		\[
			\|\theta_{\phi_\alpha}\|_\infty \lesssim NL(N+L^2).
		\]
		\item The ReLU FNN $\varphi$ approximates binomials (see considerations on \cite[Lemma 4.2]{lu_main}) for which we obtain
		\[
			\|\theta_\varphi\|_\infty\lesssim N.
		\]
	\end{itemize}
	As a consequence, the growth of parameters of $\phi$ in \cite[Theorem 2.2]{lu_main} is given by 
	\[
		\|\theta_\phi\|_\infty\lesssim \max(NL(N+L^2), c^3N^{3/d}L^{4/d}).
	\]
	The width of the realizing network fulfills $\mathcal{W}(\phi)=\mathcal{O}(q^{d+1}N\log(8N))$ and the modified depth (see the considerations on the complexity estimation in \cite[Lemma 5.6]{lu_main}) $\mathcal{D}(\phi)=\mathcal{O}(q^2L^2\log(4L))$.
	\medskip
	
	This concludes the the considerations on \cite[Theorem 2.2]{lu_main}. Next we analyze the growth of parameters of the network achieving an extension of the approximation to the whole domain based on \cite[Theorem 2.1]{lu_main}.
	\subsection[Estimate 2]{Estimation of $\|\theta_\phi\|_\infty$ for $\phi$ in \cite[Theorem 2.1]{lu_main}:} The result shows that given $f\in \mathcal{C}([0,1]^d)$ and a ReLU FNN $\tilde{\phi}$ approximating $f$ uniformly with error $\epsilon>0$ outside some trifling region \eqref{trifling_region} with respect to $\delta>0$ then there exists some ReLU FNN $\phi$ approximating $f$ uniformly with error given by $\epsilon+d\omega_f(\delta)$ where the modulus of continuity $\omega_f$ is defined as 
	\[
	\omega_f(r)=\sup\{\vert f(x)-f(y)\vert:\Vert x-y\Vert\leq r, x,y\in [0,1]^d\}
	\]
	for $r>0$. The approximating ReLU FNN $\phi$ is constructed as follows. Given $\tilde{\phi}$ approximating $f$ outside a trifling region \eqref{trifling_region} as in \cite[Theorem 2.2]{lu_main}, the networks $\phi_i$ for $0\leq i\leq d$ are set  as $\phi_0=\tilde{\phi}$ and inductively
	\[
	\phi_{i+1}(x)=\text{mid}(\phi_i(x-\delta e_{i+1}), \phi_i(x), \phi_i(x+\delta e_{i+1}))
	\]
	for $0\leq i\leq d-1$. The median function \emph{mid} is constructed by a ReLU FNN as
	\begin{align*}
	\text{mid}(x_1,x_2, x_3)=\relu(x_1+x_2+x_3)&-\relu(-x_1-x_2-x_3)\\-\max(x_1,x_2,x_3)&-\min(x_1, x_2, x_3)
	\end{align*}
	\[
	\text{where} ~ ~  \max(x,y) = \frac{1}{2}(\relu(x+y)-\relu(-x-y)+\relu(x-y)+\relu(-x+y)).
	\]
	Finally, the approximating ReLU FNN, including the trifling region, is given by $\phi=\phi_d$. As a consequence of the construction we derive immediately that
	\[
	\|\theta_{\phi_{i+1}}\|_\infty \leq \max(\|\theta_{\phi_i}\|_\infty, 1, \delta).
	\]
	Thus, resolving the recursive inequality yields
	\[
	\|\theta_{\phi}\|_\infty\leq \max(\|\theta_{\tilde \phi}\|_\infty, 1, \delta)\leq \|\theta_{\tilde \phi}\|_\infty
	\]
	where in the last inequality we have used that $0<\delta<1$ and $1\leq \|\theta_{\tilde \phi}\|_\infty$. Hence, the asymptotical behavior of the parameters of $\phi$ is governed by the suprema of the parameters of the reduced approximation $\tilde{\phi}$.\\
	
	This concludes the considerations on \cite[Theorem 2.1]{lu_main} and shows that the approximating network $\phi$ constructed in \cite{lu_main} under the previously discussed modification attains an error of order $N^{-2q/d}L^{-2q/d}+d\omega_f(\delta)$ for $\epsilon=N^{-2q/d}L^{-2q/d}$ and growth of parameters
	\[
		\|\theta_{\phi}\|_\infty\lesssim\max(NL(N+L^2), c^3N^{3/d}L^{4/d})
	\]
	with width $\mathcal{W}(\phi)=\mathcal{O}(q^{d+1}3^dN\log(8N))$ and depth $\mathcal{D}(\phi)=\mathcal{O}(q^2L^2\log(4L)+d)$.
	
	\subsection[Estimate 3]{Estimation of $\|\theta_{\phi}\|_\infty$ with $\phi$ in \cite[Theorem 1.1]{lu_main}:} 
	Finally, in the main result \cite[Theorem 1.1]{lu_main}, the trifling region is chosen small enough (i.e., the parameter $\delta>0$ is chosen sufficiently small), to recover the approximation rate $N^{-2q/d}L^{-2q/d}$ on the whole domain $[0,1]^d$ in terms of the width $N$ and depth $L$. Following the proof of \cite[Theorem 1.1]{lu_main} the parameter $0<\delta < \frac{1}{3R}$ determining the trifling region \eqref{trifling_region} is chosen small enough such that
	\[
	d\omega_f(\delta)\leq N^{-2q/d}L^{-2q/d}.
	\]
	Note that $f\in\mathcal{C}^q([0,1]^d)$ with $q\geq 1$. Denoting the Lipschitz constant of $f$ by $\tilde{L}$ we have $\omega_f(\delta)\leq \tilde{L}\delta$. As a consequence it suffices to choose 
	\begin{align}
		\label{deltachoice}
		\delta \leq d^{-1}\tilde{L}^{-1}N^{-2q/d}L^{-2q/d}.
	\end{align}
Recall that in the considered analysis of growth of parameters, the dependence on $\delta $ enters via the assumption that there exists some $c\in \mathbb{N}$ with $\delta = \frac{1}{(c+1)R}$. Here, a sufficiently small $\delta >0$ can be achieved by choosing $c\in \mathbb{N}$ sufficiently large. For that, note first that \[R=\lfloor N^{1/d}\rfloor^2\lfloor L^{2/d}\rfloor\geq (\frac{1}{2}N^{1/d})^2(\frac{1}{2}L^{2/d})= \frac{1}{8}N^{2/d}L^{2/d}.\] Thus, for \eqref{deltachoice} to hold, it suffices to choose $c\geq 3$ as the smallest integer fulfilling 
	\begin{align}
		\label{cchoice}
	8\tilde{L}dN^{2(q-1)/d}L^{2(q-1)/d}-1\leq c.
	\end{align}

	Inserting this choice of $c$ we derive that the growth of parameters of the network $\phi$ analyzed in the previous subsection realizing the approximation of $f$ with error $N^{-2q/d} L^{-2q/d}$ on the whole domain $[0,1]^d$ is given by
	\[
		\|\theta_{\phi}\|_\infty\lesssim\max(d^3\tilde{L}^3N^{\frac{6q-3}{d}}L^{\frac{6q-2}{d}}, NL(N+L^2))
	\]
	with width $\mathcal{W}(\phi)=\mathcal{O}(q^{d+1}3^dN\log(8N))$ and depth $\mathcal{D}(\phi)=\mathcal{O}(q^2L^2\log(4L)+d)$.
	\section{Conclusions}
	In this work, we have analyzed the growth of parameters realizing certain classes of feed forward ReLU neural networks as they approximate functions of given regularity with decreasing error. We showed analytically for the deep approximation result introduced in \cite{lu_main} with ReLU activation function and slightly increased depth, that the realizing parameters of the networks approximating differentiable functions grow at polynomial rate. This rate is superior than the one of existing results in most cases, depending on the input dimension and regularity of the approximated function. In particular, for high dimensional input, it is always superior. An exception is the partition of unity based approximation of \cite{guehring20}, for which we derive in Appendix \ref{sec:pou} a slower growth of the parameters, at the price of an approximation which is not of optimal order in width and depth. The construction in Theorem \ref{theorem_lu_param_growth} achieves the stated approximation-error bound with polynomially growing parameters. Controlling the parameter growth requires an additional factor $L$ in the depth scaling compared with the construction in Theorem \ref{lu_param_growth}. We do not claim joint width-depth optimality for the modified construction.
	We further point out that the immediate significance of explicit rates lies in the possibility of a direct incorporation in
an error analysis for neural network training. The comparison with the prior literature has also shown that the a priori bounded growth of the parameters obtained for certain networks with a rectified quadratic unit (ReQU) as activation function is unmatched. This, along other beneficial properties, makes ReQU networks an interesting topic of future research.

\appendix

\section[Case study of a shallow network]{Case study of shallow network in \cite{mhaskar}}
\label{subsec:shallow}
In this section we show that the norm of the parameters of the neural networks as used in the approximation result of \cite{mhaskar} grow at least exponentially when approximation a specific function as given in \eqref{special_f} below. While this negative result is not a counterexample (which would also exclude existence of a different parameter choice for the same architecture with improved rates) is still shows, for a classical choice of approximation, that a polynomial growth of the norms of the parameters cannot be taken for granted. 

To formulate our negative results, we first recall that exact result of \cite[Theorem 2.1]{mhaskar}: For $1\leq d\leq s$, $1\leq l$ and $1\leq q\leq \infty$, a function $f\in W^{l,q}([-1,1]^s)$ may be approximated by a single hidden layer feed forward neural network w.r.t. to a smooth activation function $\phi:\mathbb{R}^d\to \mathbb{R}$ if there exists a $b\in \mathbb{R}^d$ such that $\phi^{(\vert k\vert)}(b)\neq 0$ for all $k\in \mathbb{N}_0^d$. More concretely it is shown that for

\begin{align}
	\label{mhaskar_neural_network}
f_m(x) = \sum_{0\leq \underline{r}\leq \underline{p}\leq\underline{k}\leq 2m}[V_{\underline{k}}(f)\tau_{\underline{k},\underline{p}}(\phi^{(\vert \underline{p}\vert)}(b))^{-1}h^{-\vert\underline{p}\vert}(-1)^{\vert \underline{r}\vert}\binom{\underline{p}}{\underline{r}}\phi(h(2\underline{r}-\underline{p})\cdot x +b)]
\end{align}
(with $h = \min\{\frac{\delta}{3ms}, \max_{0\leq \underline{k}\leq 2m}\{m^{l+\alpha}M_{\phi;m,s}\sum_{0\leq \underline{p}\leq \underline{k}}\vert\phi^{(\vert\underline{p}\vert)}(b)\vert^{-1}\vert \tau_{\underline{k}, \underline{p}}\vert\}^{-1/2}\}$, where $\delta>0$ such that $\phi$ is smooth in $(b-\delta,b+\delta)$, $\alpha = s/\min(p,2)$, $M_{\phi;m,s}$ are certain approximation constants, $\tau_{\underline{k}, \underline{p}}$ are multidimensional Chebychev coefficients such that for $T_{\underline{k}}(x)=\prod_{j=1}^{s}T_{k_j}(x_j)$ with $T_{k_j}(2\cos t)= \cos(k_j t)$  and $T_{\underline{k}}(x):=\sum_{0\leq \underline{p}\leq \underline{k}} \tau_{\underline{k},\underline{p}}x^{\underline{p}})$ it holds true that
\[
\Vert f-f_m\Vert_{L^q([-1,1]^s)}\leq c m^{-l/s}\Vert f\Vert_{W^{l,q}([-1,1]^s)}.
\]
Furthermore, the estimate 
\begin{align*}
	\sum_{0\leq \underline{k}\leq 2m}\vert V_{\underline{k}}(f)\vert \leq cm^\alpha\Vert f\Vert_{W^{l,q}([-1,1]^s)}
\end{align*} is valid for certain coefficients $V_{\underline{k}}(f)$ of a de la Valleé Poussin type operator. The number of parameters of $f_m$ is of order $m^s$. That is, \eqref{approximation} holds with $\mathcal{X}= W^{l,q}([-1,1]^s), \mathcal{Y}=L^q([-1,1]^s)$ and $\alpha_\mathcal{X}(N,L)=N^{-l/s}$. 

For the sake of simplicity regarding the technical details we restrict ourselves to the case $s=1$ in the following considerations. Furthermore, we consider the function
\begin{align}
	\label{special_f}
f(t)=
\begin{cases}
	\pi^2/4-\arccos(t/2)^2, & \text{if } 0\leq t\leq 1,\\
	0, & \text{if } -1\leq t<0
\end{cases}
\end{align}
which is $W^{1,\infty}([-1,1])$-regular. We refer to Remark \ref{rem:generalization1} and Remark \ref{rem:generalization2} for generalizations of the case study in Lemma \ref{lem:case} to a broader class of functions.
	\begin{lemma}
		\label{lem:case}
	Let $1\leq q\leq \infty$ and $f\in W^{1,\infty}([-1,1])$ be given as in \eqref{special_f}. Then the single hidden layer feed forward neural networks $(f_m)_{m\in\mathbb{N}}$ in \eqref{mhaskar_neural_network}  with width of order $\mathcal{O}(m)$ and activation function given by either $\phi(x)=\exp(-x^2)$ or $\phi(x)=(1+\exp(-x))^{-1}$ fulfill
	\[
	\Vert f-f_m\Vert_{L^q([-1,1])}=\mathcal{O}(m^{-1}).
	\]
	Furthermore, there exists some $c>1$ such that the realizing parameters of the $(f_{m})_m$ grow asymptotically as $\Omega(c^{m})$.
\end{lemma}
To prove this, we analyze the suprema of the realizing parameters of $f_m$ in \eqref{mhaskar_neural_network},
\begin{align}
	\label{mhaskar_param_growth}
	\max_{0\leq r\leq p\leq k\leq 2m}\left(\vert V_{k}(f)\tau_{k, p}(\phi^{(p)}(b))^{-1}h^{-p}\binom{p}{r}\vert, h\vert 2r-p\vert, \vert b\vert\right).
\end{align}
As $h\vert 2r-p\vert \leq \frac{2m\delta}{3m}=\frac{2}{3}\delta$ for $0\leq r\leq p\leq2m$ the asymptotical behavior of \eqref{mhaskar_param_growth} is determined by the term
\begin{align}
	\label{mh_upper_1}
	I_m=\max_{0\leq r\leq p\leq k\leq 2m}\bigg\vert V_{k}(f)\tau_{k, p}(\phi^{(p)}(b))^{-1}h^{-p}\binom{p}{r}\bigg\vert.
\end{align}

Using $\binom{p}{r}\geq 1$ for $p\geq r$ together with boundedness of $h^{-p}$ by $(3m/\delta)^{p}$ from below and $\vert\phi^{(p)}(b)\vert\leq \max_{0\leq \tilde{p}\leq 2m}\vert\phi^{(\tilde{p})}(b)\vert$ for $0\leq p\leq 2m$ we derive that
\begin{align}
	\label{intermed_estim1}
	I_m\geq (\max_{0\leq \tilde{p}\leq 2m}\vert\phi^{(\tilde{p})}(b)\vert)^{-1}\max_{0\leq p\leq k\leq 2m}\vert V_k(f)\vert \vert \tau_{k,p}\vert (3m/\delta)^p.
\end{align}
Estimating $\vert V_k(f)\vert$ in \eqref{intermed_estim1} from below by its minimum over $0\leq k\leq 2m$ and the remaining maximum by $\vert \tau_{2m,2m}\vert (3m/\delta)^{2m}$ by inserting $p=k=2m$ we obtain
\begin{align}
	\label{estim_I}
	\frac{1}{2}(3m/\delta)^{2m}(\max_{0\leq p\leq 2m}\vert\phi^{(\vert p\vert)}(b)\vert)^{-1}\min_{0\leq k\leq 2m}\vert V_{k}(f)\vert\leq I_m
\end{align}
as $\tau_{2m,2m}=1/2$.
 Following \cite{mhaskar} and using $s=1$ it can be shown that
	\begin{align}
		\label{VKestimation}
		V_k(f)=\begin{cases}
			\hat{f^*}(0) & \text{if } k=0,\\
			2\hat{f^*}(k) & \text{if } 1\leq k\leq m,\\
			2\frac{2m-k+1}{m+1}\hat{f^*}(k) & \text{if } m+1\leq k\leq 2m
		\end{cases}
	\end{align}
	where $f^*(t) = f(2\cos t)$ for $t\in[-\pi,\pi]$ is a periodic modification of the extension of $f$ to $[-2,2]$. Depending on $f$ the minimum over the coefficients $V_{k}(f)$ on the left hand side of \eqref{estim_I} may be calculated explicitly. For $f$ as in \eqref{special_f} it follows that
	\begin{align}
		\label{special_f_mod}
		f^*(t)=
		\begin{cases}
			\pi^2/4-t^2, & \text{if } ~ t\in [-\pi/2,\pi/2]\\
			0, & \text{if } ~ t\in[-\pi,-\pi/2)\cup(\pi/2,\pi]
		\end{cases}.
	\end{align}

\begin{figure}
	\centering
\begin{tikzpicture}
	\draw[<->](-2.2,0)--(2.2,0) node[right]{$x$};
	\draw [->](0,0)--(0,2.7) node[above]{$y$};
	\foreach \x in {-pi/2, pi/2} \draw (\x,-3pt)--(\x,3pt);
	\draw (-pi/2,0) node[below]{$-\pi/2$};
	\draw (pi/2,0) node[below]{$\pi/2$};
	\draw[thick,blue] (1.4,1.7) node[below]{$f$};
	\draw[thick,red] (-1.4,1.7) node[below]{$f^*$};
	\draw[thick,blue,-] plot [domain=0:2,samples=200] (\x,{pi^2/4-rad(acos(\x/2))^2});
	\draw[thick,blue,-] plot [domain=-2:0,samples=2] (\x,0);
	
	\draw[thick,red,-] plot [domain=-pi/2:pi/2,samples=200] (\x,{pi^2/4-\x*\x});
	\draw[thick,red,-] plot [domain=-2:-pi/2,samples=2] (\x,0);
	\draw[thick,red,-] plot [domain=pi/2:2,samples=2] (\x,0);
\end{tikzpicture}
\caption{The approximated function $f$ in \eqref{special_f} and the modification $f^*$ in \eqref{special_f_mod}.}
\end{figure}
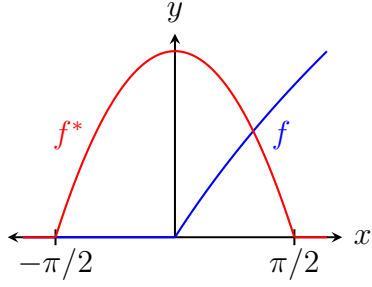
	For $\hat{f^*}(k)=\frac{1}{2\pi}\int_{-\pi}^\pi f^*(t)\exp(-ikt)\dx t$ it can be shown that $\hat{f^*}(0)=\pi^2/12$ and $\vert \hat{f^*}(k)\vert\geq k^{-3}$. As a consequence, it follows for $f$ as in \eqref{special_f} that
	\begin{align}
		\label{Vkestimation}
	\min_{0\leq k\leq 2m}\vert V_k(f)\vert \geq \frac{1}{4m^3}
	\end{align}

for $m\geq 1$. In view of \eqref{estim_I} it remains to estimate 
		\begin{align}
			\label{lowerbd}
			(3m/\delta)^{2m}(\max_{0\leq p\leq 2m}\vert\phi^{(p)}(b)\vert)^{-1}
		\end{align}
		from below. The next two lemmata show that the term in \eqref{lowerbd} may grow asymptotically exponential in $m$ for the activation function $\phi(x)=\exp(-x^2)$ or $\phi(x) = (1+\exp(-x))^{-1}$, thereby completing the proof of Lemma \ref{lem:case}.
		\begin{lemma}
			\label{gauss_activ}
			In the situation above for $\phi(x)=\exp(-x^2)$ it holds true that
			\begin{align}
				\label{gauss_asymp}
				(3m/\delta)^{2m}(\max_{0\leq p\leq 2m}\vert\phi^{( p)}(b)\vert)^{-1}=\Omega\bigg(\frac{1}{\sqrt{ m}}\bigg(\frac{3e}{2\delta}\bigg)^{2m}\bigg)
			\end{align}
			and hence, the term in \eqref{gauss_asymp} grows at least exponentially if $0<\delta < 3e/2$.
		\end{lemma}
		\begin{proof}
			First note that $\phi(x)=\exp(-x^2)$ fulfills the requirements of \cite[Theorem 2.1]{mhaskar}. Furthermore, as
			\[
				\max_{0\leq p\leq 2m}\vert\phi^{(p)}(b)\vert\leq \max_{0\leq p\leq 2m}\Vert\phi^{(p)}\Vert_{L^\infty(\mathbb{R})}
			\]
			it suffices to estimate $\Vert \phi\Vert_{W^{2m,\infty}(\mathbb{R})}$ for $m\in \mathbb{N}$. We start by noting that
			\begin{align}
				\label{hermite}
				\phi^{(p)}= (-1)^pH_p \phi
			\end{align}
			where $H_p$ denotes the $p-$th Hermite polynomial. The Hermite polynomials fulfill the recursion formula
			\begin{align}
				\label{hermite_rec}
				2x H_p(x) -2p H_{p-1}(x)=H_{p+1}(x)
			\end{align}
			for $p\geq 1$. We claim that for $p\geq 0$ it holds true that
			\[
				\Vert\phi^{(p)}\Vert_{L^\infty(\mathbb{R})}\leq p!
			\]
			We show first for $U=\mathbb{R}\backslash[-1,1]$ that $\Vert\phi^{(p)}\Vert_{L^\infty(U)}\leq p!$.
			The assertion for $p=0$ is trivial as $0\leq \phi\leq 1$. Assuming that
			\begin{align}
				\label{hypothesis}
				\Vert \phi^{(p-1)}\Vert_{L^\infty(U)}\leq (p-1)!
			\end{align}
			 we determine first the extrema of the function $\vert\phi^{(p)}\vert$ in $U$. The local extrema of $\phi^{(p)}$ in $U$ are determined by the roots of $\phi^{(p+1)}$. Using \eqref{hermite} together with positivity of $\phi$ and the recursion formula \eqref{hermite_rec} one derives that in a local extremum $x^*\in U$ of $\phi^{(p)}$ the equation
			\begin{align}
				\label{gauss_phi_step}
				\vert x^*\vert \vert \phi^{(p)}(x^*)\vert = p\vert \phi^{(p-1)}(x^*)\vert
			\end{align}
			is valid. As a consequence of \eqref{gauss_phi_step}, we derive using \eqref{hypothesis} and $\lim_{x\to\pm\infty}\phi^{(p)}(x)=0$ for $\vert x^*\vert\geq 1$ that
			\[
				\vert \phi^{(p)}(x^*)\vert\leq p!
			\]
			It remains to argue that $\vert \phi^{(p)}\vert\leq p!$ in $[-1,1]$. For $p=0$ the assertion is again trivial. For $p=1$ it is straightforward to show that the maximum of $\vert \phi'\vert$ on $[-1,1]$ is attained at $1/\sqrt{2}$ with value less than one. For $p=2, p=3$ one can show that already $\vert H_p\vert\leq p!$ yielding the assertion as $\vert\phi\vert\leq 1$. Finally, for $p\geq 4$ assuming the assertion for $p-1$ and $p-2$ we obtain for $z\in [-1,1]$ that
			\begin{align*}
				\vert \phi^{(p)}(z)\vert &= \vert H_p(z)\phi(z)\vert\\
				&=\vert \phi(z)(2z H_{p-1}(z)-2(p-1)H_{p-2}(z))\vert\\
				&=2\vert z\phi^{(p-1)}(z)+(p-1)\phi^{(p-2)}(z)\vert\\
				&\leq 2((p-1)!+(p-1)!)\\
				&=4(p-1)!\\
				&\leq p!
			\end{align*}
			as $p\geq 4$. We conclude that 
			\[
				\Vert \phi\Vert_{W^{2m,\infty}(U)}\leq (2m)!
			\]
			and thus, that
			\[
				\frac{m^{2m}}{(2m)!}\leq m^{2m}(\max_{0\leq p\leq 2m}\vert\phi^{(\vert p\vert)}(b)\vert)^{-1}.
			\]
			Employing Stirling's Formula yields $\frac{m^{2m}}{(2m)!}\sim \frac{1}{\sqrt{4\pi m}}(\frac{e}{2})^{2m}$ and finally the assertion of the lemma.
		\end{proof}
		\begin{lemma}
			\label{lem:softlog}
			In the situation above for $\phi(x) = (1+\exp(-x))^{-1}$ it holds true that
			\begin{align}
				\label{log_asymp}
			(3m/\delta)^{2m}(\max_{0\leq p\leq 2m}\vert\phi^{(\vert p\vert)}(b)\vert)^{-1}=\Omega\bigg(\frac{1}{\sqrt{ m}}\bigg(\frac{3e}{\delta}\bigg)^{2m}\bigg)
			\end{align}
			and hence, the term in \eqref{log_asymp} grows at least exponentially if $0<\delta<3e$.
		\end{lemma}
		\begin{proof}
			Similarly as in Lemma \ref{gauss_activ} we estimate $\Vert \phi\Vert_{W^{2m,\infty}(\mathbb{R})}$. It is straightforward to show that 
				\[
					\phi^{(p)}=P_p(\phi)
				\]
				for polynomials $P_p$ of degree $p+1$ fulfilling the recursion formula
				\[
					P_{p+1} = P_p' P_1
				\]
				for $p\geq 1$ where $P_1(x)=x^2-x$. As $0\leq\phi\leq 1$ the problem of estimating $\Vert \phi\Vert_{W^{2m,\infty}(\mathbb{R})}$ reduces to determining an upper bound of
					\[
						\Vert P_p\Vert_{L^\infty([0,1])}.
					\]
				Employing the general Leibniz rule we derive
				\begin{align}
					\label{recurr}
					P_n^{(k)}=\sum_{j=0}^k\binom{k}{j}P_{n-1}^{(k+1-j)}P_1^{(j)}=P_{n-1}^{(k+1)}P_1+kP_{n-1}^{(k)}P_1'+\frac{k(k-1)}{2}P_{n-1}^{(k-1)}P_1''.
				\end{align}
				Define for $n\geq1, k\geq 0$
				\begin{align}
					\label{alphadef}
				\alpha_{n,k}=\Vert P_n^{(k)}\Vert_{L^\infty([0,1])}.
				\end{align}
				Let further $\alpha_{n,-1}:=0$ for all $n\geq 0$ and $\alpha_{0,k}=\delta_{k1}$, i.e., one if $k=1$ and zero else for $k\in\mathbb{N}$. Note that $\alpha_{n,k}=0$ for $k\geq n+2$. Then by \eqref{recurr} the recursive inequality
				\begin{align}
					\label{alpha_rec}
					\alpha_{n,k}\leq \frac{1}{4}\alpha_{n-1,k+1}+k\alpha_{n-1,k}+k(k-1)\alpha_{n-1,k-1}
				\end{align}
				is valid for $1\leq n$ and $0\leq k\leq n+1$. We claim that
				\begin{align}
					\label{alpha_claim}
					\alpha_{n,k} \leq \frac{n! (n+1)!}{(n-k+1)!}2^{k-n-1}.
				\end{align}
				We prove the assertion via induction. For $n=1$ \eqref{alpha_claim} follows by comparing the right hand side of \eqref{alpha_claim} to the definition of $\alpha_{1,k}$ in \eqref{alphadef} directly. In the induction step for $n+1$ and $0\leq k\leq n+2$ we consider four cases.\\
				
				In case $1\leq k\leq n$ we obtain by \eqref{alpha_rec} and the induction hypothesis for $n$ that
				\begin{align*}
					\alpha_{n+1,k}&\leq \frac{1}{4}\alpha_{n,k+1}+k\alpha_{n,k}+k(k-1)\alpha_{n,k-1}\\
					&\leq\frac{1}{4}\frac{n!(n+1)!}{(n-k)!}2^{k-n}+k\frac{n!(n+1)!}{(n-k+1)!}2^{k-n-1}+k(k-1)\frac{n!(n+1)!}{(n-k+2)!}2^{k-n-2}\\
					&=\frac{(n+1)!(n+2)!}{(n-k+2)!}2^{k-n-2}.
				\end{align*}
				For $k=n+1$ we have similarly
				\begin{align*}
					\alpha_{n+1,n+1}&\leq (n+1)[\alpha_{n,n+1}+n\alpha_{n,n}]\\
					&\leq (n+1)[n!(n+1)!+n!(n+1)!\frac{n}{2}]\\
					&=\frac{1}{2}(n+1)!(n+2)!
				\end{align*}
				For $k=n+2$ we derive
				\begin{align*}
					\alpha_{n+1,n+2}&\leq (n+1)(n+2)\alpha_{n,n+1}\\
					&\leq (n+1)!(n+2)!
				\end{align*}
				Finally for $k=0$ it holds true that
				\begin{align*}
					\alpha_{n+1, 0}&\leq \frac{1}{4}\alpha_{n,1}\\
					&\leq (n+1)!2^{-n-2}
				\end{align*}
				Thus, the assertion in \eqref{alpha_claim} is valid. As a consequence
				\[
					\Vert\phi\Vert_{W^{2m,\infty}(\mathbb{R})}=\Vert P_{2m}\Vert_{L^\infty([0,1])}=\alpha_{2m,0}\leq\frac{(2m)!}{2^{2m+1}}.
				\]
				As in Lemma \ref{gauss_activ} by employing Stirling's formula it follows that
				\[
					(3m/\delta)^{2m}(\max_{0\leq p\leq 2m}\phi^{(\vert p\vert)}(b))^{-1}=\Omega\bigg(\frac{1}{\sqrt{ m}}\bigg(\frac{3e}{\delta}\bigg)^{2m}\bigg).\qedhere
				\]
		\end{proof}
		\begin{proof}[Proof of Lemma \ref{lem:case}]
			By the considerations in Lemma \ref{gauss_activ} for $\phi(x)=\exp(-x^2)$ and Lemma \ref{lem:softlog} for $\phi(x)=(1+\exp(-x))^{-1}$ together with the estimations \eqref{estim_I} and \eqref{Vkestimation} it follows that there exists some $\tilde{c}>0$ (depending on $\delta>0$) such that $I_m$ is bounded from below by
			\begin{align*}
				I_m \geq \frac{1}{8}\tilde{c}^mm^{-3}
			\end{align*}
			for sufficiently large $m\in\mathbb{N}$ where we make use of the exponential lower bound of the term in \eqref{lowerbd}. Hence, for $1<c<\tilde{c}$ we derive that the realizing parameters of $(f_m)_m$ grow asymptotically as $\Omega(c^m)$. Employing \cite[Theorem 2.1]{mhaskar} the remaining assertions of the Theorem follow.
		\end{proof}
		\begin{remark}
			\label{rem:generalization1}
			We recall that in view of Lemma \ref{gauss_activ} and Lemma \ref{lem:softlog} the realizing parameters of the networks $f_m$ are essentially bounded from below by
			\begin{align}
				\label{lowerbdparam}
				c^m \min_{0\leq k\leq 2m}\vert V_k(f)\vert
			\end{align}
			for some $c>1$. For $f$ as in \eqref{special_f} we showed that $\min_{0\leq k\leq 2m}\vert V_k(f)\vert\geq \frac{1}{4} m^{-3}$. More generally, if there exists some $\tilde{c}>0$ and $\beta\in \mathbb{N}_0$ such that
			\begin{align}
				\label{Vk}
				\min_{0\leq k\leq 2m}\vert V_k(f)\vert \geq \tilde{c} m^{-\beta},
			\end{align}
			then the realizing parameters of $f_m$ grow exponentially due to the lower bound in \eqref{lowerbdparam}. By \eqref{VKestimation} the bound in \eqref{Vk} is equivalent to $\vert \hat{f^*}(m)\vert\geq \tilde{c}m^{-\beta}$ for some $\tilde{c}>0, \beta\in\mathbb{N}_0$ with $\hat{f^*}(m)$ the Fourier coefficients of $f^*$. As a consequence of \cite[Section 4, Theorem 4.1, Exercise 4.1]{Katznelson2004} for a given zero sequence $(a_m)_m$ with $a_m \geq \tilde{c}m^{-\beta}$ there exists a function of certain regularity depending on $(a_m)_m$ whose Fourier coefficients are exactly the $a_m$. Hence, the negative result of Lemma \ref{lem:case} can be recovered for a much broader class of functions $f^*$ and hence, also $f$.
		\end{remark}
		\begin{remark}
			\label{rem:generalization2}
			In the following we consider a general class of functions with direct conditions on $f$ rather than $f^*$ as in Remark \ref{rem:generalization1}. Under the assumption that $f\in \mathcal{C}^{l-1}([-2,2])$ with $f^{(l-1)}$ absolutely continuous and $f^{(l)}$ discontinuous, we can argue
that \eqref{Vk} holds true at least for a subsequence of increasing natural numbers $m$ as follows.
First note that, for $f^*(t)=f(2\cos t)$, it immediately follows that $f^*\in \mathcal{C}^{l-1}([-\pi,\pi])$ with $(f^*)^{(l-1)}$ absolutely continuous and $(f^*)^{(l)}$ discontinuous. From \cite[4.4 Theorem]{Katznelson2004} it then follows that $\vert \hat{f^*}(m)\vert = o(m^{-l})$.  Conversely, by \cite[Exercise 4.2]{Katznelson2004}, for $g\in L^1((-\pi,\pi))$ such that $\vert\hat{g}(m)\vert=\mathcal{O}(m^{-(l+1+\epsilon)})$ for some $\epsilon>0$ it follows that $g\in \mathcal{C}^l((-\pi,\pi))$. Thus, as $f^{*(l)}$ is discontinuous, there exists a subsequence $(m_i)_i\subseteq \mathbb{N}$ fulfilling
			\[
				\vert \hat{f^*}(m_i)\vert > \frac{1}{m_i^{l+2}}
			\]
			for all $i\in \mathbb{N}$, which immediately yields a polynomial lower bound on the coefficients $V_k(f)$ in \eqref{VKestimation} for this subsequence. Setting
			\[
				m(m_i)=\max\{m_i~|~m_i\leq 2m\}
			\]
			we can then estimate $I_m$ in \eqref{mh_upper_1}, in a slightly different way, from below by 
			\[
				I_m\geq \bigg\vert V_{m(m_i)}(f)(\phi^{(m(m_i))}(b))^{-1}\left(\frac{3m(m_i)}{2\delta}\right)^{m(m_i)}\bigg\vert.
			\]
			By similar techniques as before and applying Lemma \ref{gauss_activ} and Lemma \ref{lem:softlog} we derive that $I_{m_i}$ grows exponentially as $i\to \infty$, i.e., we recover exponential growth of the parameters in the above setup at least for a subsequence of increasing architectures.
\end{remark}
\begin{remark}
	We stress again that the result in Lemma \ref{lem:case} is not a \textit{counterexample} for shallow neural networks as it is restricted to the construction in \cite{mhaskar}. For a counterexample one would have to provide a function $f$ such that every single hidden layer neural network with Gaussian or logistic activation function satisfying the constraints on the width and approximation error has exponentially growing weights. This is out of scope here (and probably infeasible to show).
\end{remark}

\section[Partition of unity approach]{Growth of parameters for a partition of unity based approximation}
\label{sec:pou}
In this section we analyze the asymptotic behavior of the norm of the parameters realizing the approximating neural networks constructed in \cite{guehring20}. In contrast to the construction of \cite{lu_main} analyzed in Section \ref{sec:param_growth}, where the localization on the cells of a partition of $[0,1]^d$ into cubes is the piecewise-constant projection of each cell onto one of its corners, which a ReLU FNN can reproduce only outside the trifling region \eqref{trifling_region} and at the cost of a large slope across it, the localization in \cite{guehring20} is realized by a \emph{partition of unity}, that is, by a family of continuous, piecewise affine bump functions which are multiplied with local polynomial approximations of the approximated function. As the members of this partition are themselves continuous and piecewise affine, no trifling region occurs, and the extension step of \cite[Theorem 2.1]{lu_main} analyzed in Section \ref{sec:param_growth} has no counterpart. The multiplication of the bump functions with the local polynomials, on the other hand, has to be carried out by ReLU networks approximating the product, which is the same building block as the one entering the construction of \cite{lu_main}.

The approximation result of \cite{guehring20} does not include a bound on the magnitude of the realizing parameters. It is, however, pointed out in \cite[Section 5]{guehring20} that the parameters obtained in the underlying construction are unbounded as the approximation error tends to zero. We successively estimate the norm of the parameters of the constructions subnetworks, and then combine these estimates.

\paragraph{Notation.} Throughout, $q\in \mathbb{N}$ with $q\geq 2$ denotes the regularity of the approximated function, $1\leq p\leq \infty$ and $0\leq s\leq 1$ the integrability and the regularity of the norm $\Vert \cdot\Vert_{W^{s,p}([0,1]^d)}$ in which the approximation error is measured, and
\[
	\mathcal{F}_{q,d,p,B}:=\left\{f\in W^{q,p}([0,1]^d): \Vert f\Vert_{W^{q,p}([0,1]^d)}\leq B\right\}
\]
for $B>0$, denotes the class of approximated functions. Note that $\mathcal{C}^q([0,1]^d)\subseteq W^{q,p}([0,1]^d)$ for all $1\leq p \leq \infty$, so that the case $s=0$, $p=\infty$ corresponds to the setting of Section \ref{sec:param_growth}. The number of cells per coordinate direction of the underlying partition of $[0,1]^d$ is denoted by $R\in \mathbb{N}$, in accordance with the notation of \eqref{trifling_region}, and for a network $\phi\in\mathcal{FNN}$ we denote by $\mathcal{M}(\phi)$ the number of its nonzero weights and, if $\phi$ has $n$-dimensional output, by $[\phi]_l$ its $l$-th output component for $1\leq l\leq n$. For $k\in \mathbb{N}_0$ and a domain $\Omega\subseteq\mathbb{R}^d$ note $\Vert u\Vert_{W^{k,\infty}(\Omega)}=\max_{\vert\beta\vert\leq k}\Vert D^\beta u\Vert_{L^\infty(\Omega)}$. Finally, we note that the networks in \cite{guehring20} are allowed to admit skip connections, i.e., the affine map defining the $l$-th layer may act on the concatenation of the outputs of all preceding layers. By \cite[Lemma 2.13]{guehring20} any such network may be rewritten as a network in the sense of Definition \ref{def:network} realizing the same function, at the cost of an additional logarithmic factor in the number of weights and neurons. Since this transformation only requires the realization of identity mappings via $x = \relu(x)-\relu(-x)$, i.e., subnetworks all of whose parameters are $\pm 1$, it does not affect the asymptotic behavior of $\|\theta_\phi\|_\infty$. Hence, we may work with the construction as provided in \cite{guehring20}. Its approximation result is given as follows.
\begin{theorem}{\cite[Theorem 4.1]{guehring20}}
	\label{thm:guehring}
	Let $d\in \mathbb{N}$, $q\in \mathbb{N}$ with $q\geq 2$, $1\leq p\leq \infty$, $B>0$ and $0\leq s\leq 1$. Then for any $\epsilon\in (0,1/2)$ there exists an architecture with $d$-dimensional input and one-dimensional output such that for any $f\in \mathcal{F}_{q,d,p,B}$ there is a ReLU FNN $f_\epsilon$ with this architecture such that
	\[
		\Vert f_\epsilon - f\Vert_{W^{s,p}([0,1]^d)}\leq \epsilon,
	\]
	\[
		\mathcal{D}(f_\epsilon)\lesssim \log_2(\epsilon^{-q/(q-s)})\qquad\text{and}\qquad \mathcal{M}(f_\epsilon)\lesssim \epsilon^{-d/(q-s)}\log_2(\epsilon^{-q/(q-s)}),
	\]
	with the constants depending only on $d,q,p,s$ and $B$.
\end{theorem}
Based on this approximation result, the main result of this section is as follows.
\begin{proposition}
	\label{prop:pou_param_growth}
	In the setup of Theorem \ref{thm:guehring} let $f_\epsilon$ be the ReLU FNN constructed in the proof of \cite[Theorem 4.1]{guehring20}. Then the parameters $\theta_{f_\epsilon}$ of $f_\epsilon$ fulfill
	\[
		\|\theta_{f_\epsilon}\|_\infty\lesssim \epsilon^{-\frac{1}{q-s}\max\left\{1,\frac{d}{p}\right\}},
	\]
	with the constant depending only on $d,q,p,s$ and $B$. In particular, for $s=0$ and $p=\infty$, i.e., for approximation in the supremum norm of functions $f\in \mathcal{C}^q([0,1]^d)$ with $\Vert f\Vert_{\mathcal{C}^q([0,1]^d)}\leq B$, it holds that
	\[
		\|\theta_{f_\epsilon}\|_\infty\lesssim \epsilon^{-1/q}.
	\]
\end{proposition}
In order to prove Proposition \ref{prop:pou_param_growth} we follow the construction of the proof of Theorem~\ref{thm:guehring}, which is composed of the following six auxiliary results.
\begin{enumerate}[leftmargin=0.5cm]
	\item The result \cite[Proposition C.1]{guehring20} provides a ReLU FNN approximating the square function $x\mapsto x^2$ on $[0,1]$ with respect to the $W^{1,\infty}$-norm.
	\item Via the polarization identity, this yields a ReLU FNN approximating the multiplication $(x,y)\mapsto xy$ on a square $[-M,M]^2$, see \cite[Proposition C.2]{guehring20}.
	\item The result \cite[Lemma C.3]{guehring20} provides a partition of unity subordinate to a partition of $[0,1]^d$ into $(R+1)^d$ cells, whose members are products of $d$ univariate, piecewise affine functions, each realized exactly by a ReLU FNN.
	\item The result \cite[Lemma C.4]{guehring20} provides, based on averaged Taylor polynomials and the Bramble-Hilbert lemma, an approximation of $f$ by a sum of localized polynomials, together with bounds on the coefficients of these polynomials.
	\item Based on 1.\ and 2., the result \cite[Lemma C.5]{guehring20} provides a ReLU FNN approximating the product of the output components of a given network, which is applied to the localized monomials of 3.\ and 4.
	\item The result \cite[Lemma C.6]{guehring20} finally provides a ReLU FNN approximating the sum of the localized polynomials of 4., the coefficients of which enter as the parameters of the output layer.
\end{enumerate}
Accordingly, we estimate $\|\theta_\phi\|_\infty$ for the networks $\phi$ of the results 1.\ - 6.\ successively in the following paragraphs, and then conclude the proof of Proposition \ref{prop:pou_param_growth} by inserting the choice of the parameters made in the proof of Theorem \ref{thm:guehring}.

\paragraph{Complexity estimation in \cite[Proposition C.1]{guehring20}:} The result in \cite[Proposition C.1]{guehring20} shows that the function $x\mapsto x^2$ may be approximated on the unit interval with error $\epsilon$ with respect to the $W^{1,\infty}$-norm by a ReLU FNN $\phi$ with depth $\mathcal{D}(\phi)=\mathcal{O}(\log_2(\epsilon^{-1}))$ and $\mathcal{O}(\log_2(\epsilon^{-1}))$ neurons. The network $\phi$ is given by
\[
	\phi(x)=x-\sum_{i=1}^{L}2^{-2i}T_i(x),\qquad L=\lceil \log_2(\epsilon^{-1})\rceil,
\]
with $T_i$ denoting the sawtooth functions of the considerations on \cite[Lemma 5.1]{lu_main} in Section \ref{sec:param_growth} (see also \cite[Proposition 2]{yarotsky17error}). Each hidden layer of $\phi$ consists of the at most three neurons realizing one application of $T_1$ together with one neuron propagating $x-\sum_{i=1}^{n}2^{-2i}T_i(x)$ for $1\leq n\leq L$. These are the piecewise affine interpolants of $x\mapsto x^2$ at the $2^n+1$ equidistant nodes of $[0,1]$. The $\mathcal{O}(\log_2(\epsilon^{-1}))$ neurons of $\phi$ are distributed over its layers with constant width, that is, $\mathcal{W}(\phi)=\mathcal{O}(1)$. Moreover, $\phi$ is the network of \cite[Lemma 5.1]{lu_main} for the choice $N=2$ and depth $L$, in which case $k=1$ in the notation used there. Hence, it holds true that
\[
	\|\theta_\phi\|_\infty \lesssim 1.
\]

\paragraph{Complexity estimation in \cite[Proposition C.2]{guehring20}:} The result in \cite[Proposition C.2]{guehring20} shows that for any $M\geq 1$ the function $(x,y)\mapsto xy$ may be approximated on the square $[-M,M]^2$ with error $\epsilon$ with respect to the $W^{1,\infty}$-norm by a ReLU FNN $\phi$ with $\mathcal{D}(\phi)=\mathcal{O}(\log_2(\epsilon^{-1}))$ and $\mathcal{O}(\log_2(\epsilon^{-1}))$ neurons, where the constants depend on $M$. For $\psi$ denoting the network providing the result of the previous paragraph with accuracy $\delta:=\epsilon/(6M^2C)$, where $C>0$ is independent of $\epsilon$ and $M$, the network $\phi$ is defined by
\[
	\phi(x,y)=2M^2\left(\psi\left(\frac{\vert x+y\vert}{2M}\right)-\psi\left(\frac{\vert x\vert}{2M}\right)-\psi\left(\frac{\vert y\vert}{2M}\right)\right)
\]
with $\vert z\vert=\relu(z)+\relu(-z)$. This is the same construction as the one in \cite[Lemma 4.2]{lu_main}, the polarization identity being applied on $[-M,M]^2$. Besides the parameters of the three copies of $\psi$, the network $\phi$ consists of the parameters $\pm 1$ and $\pm\frac{1}{2M}$ of the affine maps preceding these copies and of the parameters $\pm 2M^2$ of its output layer. Thus, we obtain by the previous paragraph that
\[
	\|\theta_\phi\|_\infty \lesssim M^2.
\]
Since $\phi$ consists of three copies of $\psi$ in parallel, the previous paragraph moreover yields $\mathcal{W}(\phi)=\mathcal{O}(1)$. In particular, the accuracy $\delta$ required of the square network enters only the depth of $\phi$, so that both $\mathcal{W}(\phi)$ and $\|\theta_\phi\|_\infty$ are independent of $\epsilon$.

\paragraph{Complexity estimation in \cite[Lemma C.3]{guehring20}:} The result in \cite[Lemma C.3]{guehring20} provides a partition of unity on $[0,1]^d$ subordinate to the cells $$\Omega_{m,R}:=\left\{x\in \mathbb{R}^d\,:\,\max_{1\leq l\leq d}\left\vert x_l-\frac{m_l}{R}\right\vert<\frac{1}{R}\right\}$$ for $m\in \left\{0,\dots, R\right\}^d$, whose members are products of $d$ univariate functions each of which is realized exactly by a ReLU FNN. More concretely, let $\psi$ be the single hidden layer ReLU FNN with $\mathcal{W}(\psi)=4$ given by
\[
	\psi(z)=\relu(z+2)-\relu(z+1)-\relu(z-1)+\relu(z-2),
\]
i.e., $\psi(z)=1$ for $\vert z\vert \leq 1$, $\psi(z)=2-\vert z\vert$ for $1\leq \vert z\vert\leq 2$ and $\psi(z)=0$ for $\vert z\vert\geq 2$, and let the ReLU FNN $\Psi_m$ with $d$-dimensional input and output be defined by
\[
	\Psi_m(x)=(\psi(3Rx_1-3m_1), \dots, \psi(3Rx_d-3m_d))^T.
\]
Then the members of the partition of unity are given by $\eta_m(x)=\prod_{l=1}^d[\Psi_m(x)]_l$ for $x\in [0,1]^d$ and $m\in \left\{0,\dots, R\right\}^d$, and they fulfill $0\leq \eta_m\leq 1$, $\sum_m\eta_m=1$ on $[0,1]^d$ and $\supp~\eta_m\subseteq \Omega_{m,R}$. Composing the affine maps $x\mapsto 3Rx_l-3m_l$ into the hidden layer of the corresponding copy of $\psi$, the network $\Psi_m$ has the weights $3R$ and the biases $-3m_l\pm 1$ and $-3m_l\pm 2$ in its hidden layer, where $\vert 3m_l\vert\leq 3R$, and the weights $\pm 1$ in its output layer. As a consequence,
\[
	\|\theta_{\Psi_m}\|_\infty \leq 3R+2.
\]
This estimate is to be compared with the corresponding localization in Section~\ref{sec:param_growth}. There, the projection of the subcubes to their corners is piecewise constant, and its realization by a ReLU FNN requires transitions of slope of order $c$ across the trifling region, with $c$ as in \eqref{cchoice}, resulting in the bound $\lesssim c^3N^{3/d}L^{4/d}$ of the considerations on \cite[Proposition 4.3]{lu_main}, which is the dominating term in Theorem~\ref{lu_param_growth}. The localization by a partition of unity, in contrast, is continuous, and the slope required of $\psi$ is $3R$, i.e., the width reciprocal of the overlap of neighboring cells.

\paragraph{Complexity estimation in \cite[Lemma C.4]{guehring20}:} The result in \cite[Lemma C.4]{guehring20} does not involve neural networks. It shows that for $f\in W^{q,p}([0,1]^d)$ there exist polynomials $p_m(x)=\sum_{\vert\alpha\vert\leq q-1}c_{m,\alpha}x^\alpha$ for $m\in \left\{0,\dots, R\right\}^d$ such that, with $\eta_m$ denoting the partition of unity of the previous paragraph,
\[
	\Big\Vert f-\sum_{m\in \left\{0,\dots, R\right\}^d}\eta_m p_m\Big\Vert_{W^{s,p}([0,1]^d)}\leq CR^{-(q-s)}\Vert f\Vert_{W^{q,p}([0,1]^d)}
\]
for a constant $C=C(d,q,p)>0$. The polynomials $p_m$ are averaged Taylor polynomials on the cells $\Omega_{m,R}$ of an extension $\tilde{f}=Ef\in W^{q,p}(\mathbb{R}^d)$ of $f$, and the estimate is a consequence of the Bramble-Hilbert lemma \cite[Lemma B.12]{guehring20}. What is relevant for the growth of parameters is the accompanying bound on the coefficients of $p_m$, which by \cite[Lemma B.9]{guehring20} is given by
\[
	\vert c_{m,\alpha}\vert \leq \tilde{c}\Vert \tilde{f}\Vert_{W^{q,p}(\Omega_{m,R})}R^{d/p}
\]
for a constant $\tilde{c}=\tilde{c}(d,q)>0$. Since the extension operator $E$ is bounded, it holds that $\Vert \tilde{f}\Vert_{W^{q,p}(\Omega_{m,R})}\leq \Vert \tilde{f}\Vert_{W^{q,p}(\mathbb{R}^d)}\leq C_EB$ for $f\in \mathcal{F}_{q,d,p,B}$, with $C_E=C_E(d,q,p)>0$ denoting the norm of $E$, and consequently
\[
	\vert c_{m,\alpha}\vert \lesssim BR^{d/p}.
\]
These coefficients enter the construction as the parameters of the output layer of the approximating network, see the considerations on \cite[Lemma C.6]{guehring20} below. Note that for $p=\infty$ the coefficients are bounded independently of $R$.

\paragraph{Complexity estimation in \cite[Lemma C.5]{guehring20}:} The result in \cite[Lemma C.5]{guehring20} shows that the product of the output components of a given network may be approximated by a ReLU FNN. More concretely, let $n\in \mathbb{N}$, $\kappa\geq 1$, $\epsilon\in (0,1/2)$ and let $\varphi$ be a ReLU FNN with $d$-dimensional input and $n$-dimensional output whose depth, number of neurons and number of nonzero weights are bounded by $K$, and whose output components fulfill $\Vert [\varphi]_l\Vert_{W^{k,\infty}([0,1]^d)}\leq \kappa^k$ for $k\in \left\{0,1\right\}$ and $1\leq l\leq n$. Then there exists a ReLU FNN $\phi$ with $\mathcal{D}(\phi)=\mathcal{O}(K\log_2(\epsilon^{-1}))$ and $\mathcal{M}(\phi)=\mathcal{O}(K\log_2(\epsilon^{-1}))$ such that
\[
	\Big\Vert \phi-\prod_{l=1}^n[\varphi]_l\Big\Vert_{W^{k,\infty}([0,1]^d)}\lesssim \kappa^k\epsilon\qquad\text{for } k\in \left\{0,1\right\}.
\]
Here, as well as in the bounds on $\mathcal{D}(\phi)$ and $\mathcal{M}(\phi)$, the implicit constants depend on $d$ and on $n$, but neither on $\kappa$ nor on $K$ or $\epsilon$. For $\psi_M$ denoting the network approximating the product on the square $[-M,M]^2$ as discussed above, the networks $\phi_j$ are recursively defined by $\phi_1:=[\varphi]_1$ and
\[
	\phi_{j+1}(x):=\psi_{M_j}(\phi_j(x), [\varphi]_{j+1}(x)),\qquad M_j:=j+1,
\]
for $1\leq j\leq n-1$ and $x\in [0,1]^d$, and $\phi$ is defined by $\phi:=\phi_n$. Three properties of this recursion are relevant for the growth of parameters. First, the first $\mathcal{D}(\varphi)-1$ layers of $\phi_j$ and of $\varphi$ coincide for every $j$, and the row of the output layer of $\varphi$ defining $[\varphi]_{j+1}$ is reinserted unchanged into a later layer of $\phi_{j+1}$, which is possible since the networks of \cite{guehring20} admit skip connections. The parameters of $\varphi$ hence enter $\phi$ unchanged up to sign. Second, the composition of $\psi_{M_j}$ with $(\phi_j, [\varphi]_{j+1})$ is realized by the \emph{sparse concatenation} of \cite[Definition 2.8]{guehring20}, in which the two composed networks are joined by an intermediate layer realizing the identity via $x=\relu(x)-\relu(-x)$. In contrast to merging the two adjoining affine maps into a single one, which would multiply their parameters, the layers of both networks are kept and only the last layer of the right factor is duplicated with a sign flip, see \cite[Remark 2.9]{guehring20}. The magnitude of the parameters of a sparse concatenation is thus the maximum of the magnitudes of the parameters of its two factors. Third, the assumption on the output components of $\varphi$ for $k=0$ gives $\Vert [\varphi]_l\Vert_{L^\infty([0,1]^d)}\leq 1$, while the estimate for $k=0$ at the $j$-th step of the induction underlying the recursion is given by $\Vert \phi_j-\prod_{l=1}^j[\varphi]_l\Vert_{L^\infty([0,1]^d)}\leq j\epsilon$, so that $\Vert \phi_j\Vert_{L^\infty([0,1]^d)}\leq j\epsilon+1\leq j+1$. The multiplication network applied in the $j$-th step therefore indeed covers the range of its arguments for the choice $M_j=j+1\leq n$. Thus, by the considerations on \cite[Proposition C.2]{guehring20} above, all parameters generated in the recursion are of magnitude at most of order $n^2$, and we obtain
\[
	\|\theta_\phi\|_\infty \lesssim \max(\|\theta_\varphi\|_\infty, n^2).
\]
The $(n-1)$-fold composition of multiplication networks hence does not increase the magnitude of the parameters beyond that of a single multiplication network on $[-n,n]^2$. Note that the factor $n^2$ is an artifact of the choice $M_j=j+1$ made in \cite{guehring20}, which is uniform in $\epsilon\in (0,1/2)$. Indeed, since $\Vert \phi_j\Vert_{L^\infty([0,1]^d)}\leq j\epsilon+1$, the choice $M_j=2$ is admissible for all $j$ as soon as $\epsilon\leq n^{-1}$, in which case the recursion contributes parameters of magnitude $\mathcal{O}(1)$ only. As $n$ is fixed while $\epsilon\to 0$ in the application of the result in the next paragraph, this improves the dependence of the bounds obtained there on $q$ and $d$, but does not affect their asymptotic behavior.

\paragraph{Complexity estimation in \cite[Lemma C.6]{guehring20}:} The result in \cite[Lemma C.6]{guehring20} shows that for any $\epsilon\in (0,1/2)$ there exists an architecture with $\mathcal{O}(\log_2(\epsilon^{-1}))$ layers and $\mathcal{O}((R+1)^d\log_2(\epsilon^{-1}))$ nonzero weights and neurons such that for any $f\in W^{q,p}([0,1]^d)$ there exists a ReLU FNN $\phi$ with this architecture fulfilling
\[
	\Big\Vert \sum_{m\in \left\{0,\dots, R\right\}^d}\eta_m p_m-\phi\Big\Vert_{W^{s,p}([0,1]^d)}\lesssim \Vert f\Vert_{W^{q,p}([0,1]^d)}R^s\epsilon,
\]
with $\eta_m$ and $p_m$ as in the two previous paragraphs, where the implicit constants depend on $q$ and $d$, and, in the approximation error, in addition on $p$ and $s$, but neither on $R$ nor on $\epsilon$ or $f$. In particular, the architecture does not depend on $f$, which enters only through the parameters of the output layer. The network $\phi$ is constructed in two steps. First, for $\vert\alpha\vert \leq q-1$ let $\mathcal{L}_\alpha:\mathbb{R}^d\to\mathbb{R}^{\vert\alpha\vert}$ denote the linear map realizing the entrywise replication of $x$ with respect to $\alpha$, that is, the map denoted by $\mathcal{L}$ in the considerations on \cite[Proposition 4.1]{lu_main}, but without the padding by ones required there, so that $x^\alpha=\prod_{l=1}^{\vert\alpha\vert}[\mathcal{L}_\alpha(x)]_l$, with the convention that the empty product equals $1$ in the case $\alpha=0$, and each nonzero scalar parameter of $\mathcal{L}_\alpha$ is equal to $1$. With $\Psi_m$ denoting the network of the third paragraph above, the ReLU FNN $\varphi_{m,\alpha}$ with $(\vert\alpha\vert+d)$-dimensional output defined by
\[
	\varphi_{m,\alpha}(x):=(\Psi_m(x), \mathcal{L}_\alpha(x))
\]
then fulfills $\prod_{l=1}^{\vert\alpha\vert+d}[\varphi_{m,\alpha}(x)]_l=\eta_m(x)x^\alpha$ as well as $\Vert [\varphi_{m,\alpha}]_l\Vert_{W^{k,\infty}([0,1]^d)}\leq (3R)^k$ for $k\in \left\{0,1\right\}$ and $1\leq l\leq \vert\alpha\vert+d$, and, as arranging two networks in parallel corresponds to placing their parameters in disjoint blocks of the layers of the resulting network and does not change their magnitude,
\[
	\|\theta_{\varphi_{m,\alpha}}\|_\infty \leq 3R+2.
\]
The localized monomials $\eta_m x^\alpha$ are then approximated by the networks $\phi_{m,\alpha}$ provided by the previous paragraph applied to $\varphi_{m,\alpha}$, i.e., for the choice $\kappa=3R$ and $n=\vert\alpha\vert+d\leq q+d-1$, for which we obtain
\[
	\|\theta_{\phi_{m,\alpha}}\|_\infty \lesssim \max(R, (q+d)^2).
\]
Second, the sum of the localized polynomials is realized by a single affine output layer, which is sparsely concatenated with the parallelization of the networks $\phi_{m,\alpha}$, i.e., the network $\phi$ is given for $x\in [0,1]^d$ by
\[
	\phi(x)=\sum_{m\in \left\{0,\dots, R\right\}^d}\sum_{\vert\alpha\vert\leq q-1}c_{m,\alpha}\phi_{m,\alpha}(x),
\]
the coefficients $c_{m,\alpha}$ of the considerations on \cite[Lemma C.4]{guehring20} above being the parameters of this output layer. Together with the bound on $\vert c_{m,\alpha}\vert$, we derive
\[
	\|\theta_\phi\|_\infty \lesssim \max(R, (q+d)^2, BR^{d/p}).
\]
This completes the estimation of the parameters of the subnetworks the construction in Theorem \ref{thm:guehring} is composed of, and we can conclude as follows.

\begin{proof}[Proof of Proposition \ref{prop:pou_param_growth}]
	Let $\epsilon\in (0,1/2)$ and let $f\in \mathcal{F}_{q,d,p,B}$. In the proof of \cite[Theorem 4.1]{guehring20} the approximating network $f_\epsilon$ is obtained as the network $\phi$ of the considerations on \cite[Lemma C.6]{guehring20} above for the choice
	\[
		R=\Bigg\lceil \bigg(\frac{\epsilon}{2CB}\bigg)^{-\frac{1}{q-s}}\Bigg\rceil\qquad\text{and}\qquad \tilde{\epsilon}=\frac{\epsilon^{\frac{q}{q-s}}}{8CB^2C_1}
	\]
	of the number of cells per coordinate direction and of the accuracy required of the subnetworks, respectively. Here $C=C(d,q,p)>0$ denotes the constant of the considerations on \cite[Lemma C.4]{guehring20} above, for which $CB\geq 1$ may be assumed without loss of generality, and $C_1=C_1(d,q,p,s)>0$ denotes the constant in the approximation error of \cite[Lemma C.6]{guehring20}. With this choice, $\sum_m\eta_mp_m$ approximates $f$ and $f_\epsilon$ approximates $\sum_m\eta_mp_m$ up to the error $\epsilon/2$ in $W^{s,p}([0,1]^d)$ each, which yields the approximation error asserted in Theorem \ref{thm:guehring}.

	The accuracy $\tilde{\epsilon}$ enters the estimates of the preceding paragraphs only through the number of layers and weights of the subnetworks approximating the square, the multiplication and the products of output components, that is, through terms of the order $\log_2(\tilde{\epsilon}^{-1})\lesssim \log_2(\epsilon^{-q/(q-s)})$. The only $\epsilon$-dependent quantity in the bound
	\[
		\|\theta_{f_\epsilon}\|_\infty\lesssim \max\big(R, (q+d)^2, BR^{d/p}\big)
	\]
	of the previous paragraph is $R$. Since $CB\geq 1$ and $\epsilon<1/2$, it holds
	\[
		1\leq R\leq (2CB)^{\frac{1}{q-s}}\epsilon^{-\frac{1}{q-s}}+1\lesssim \epsilon^{-\frac{1}{q-s}},
	\]
	so that, $(q+d)^2$ being a constant and $\epsilon^{-1/(q-s)}\geq 1$, we arrive at
	\[
		\|\theta_{f_\epsilon}\|_\infty\lesssim \max\Big(\epsilon^{-\frac{1}{q-s}}, \epsilon^{-\frac{d}{p(q-s)}}\Big)=\epsilon^{-\frac{1}{q-s}\max\left\{1,\frac{d}{p}\right\}},
	\]
	with the constant depending only on $d,q,p,s$ and $B$, proving the first assertion. For $s=0$ and $p=\infty$ we have $\max\left\{1,d/p\right\}=1$ and hence $\|\theta_{f_\epsilon}\|_\infty\lesssim \epsilon^{-1/q}$. Since $\Vert f\Vert_{W^{q,\infty}([0,1]^d)}\leq \Vert f\Vert_{\mathcal{C}^q([0,1]^d)}$, the class $\mathcal{F}_{q,d,\infty,B}$ contains all $f\in \mathcal{C}^q([0,1]^d)$ with $\Vert f\Vert_{\mathcal{C}^q([0,1]^d)}\leq B$, which gives the second assertion.
\end{proof}
\begin{remark}
	\label{rem:pou_regimes}
	In Proposition \ref{prop:pou_param_growth}, for $p\geq d$, and in particular for $s=0$ and $p=\infty$, the dominating term comes from the partition of unity, whose parameters are of the order $3R$ of the reciprocal overlap of neighboring cells. For $p<d$ the coefficients $c_{m,\alpha}\lesssim BR^{d/p}$ of the output layer dominate. The regularity $s$ enters both terms in the same way through $R\approx \epsilon^{-1/(q-s)}$.
\end{remark}
\begin{remark}
	\label{rem:raslan_comparison}
	Proposition \ref{prop:pou_param_growth} does not contradict \cite[Proposition 4.8]{raslan21} in Table \ref{comparison:2}, where an approximation based on a partition of unity as well is obtained with parameters of order $\mathcal{O}(\epsilon^{-(2+3d/q)})$. The construction there admits general activation functions $\varrho$ which are smooth in a neighborhood of some $x_0$ with $\varrho''(x_0)\neq 0$, and realizes the multiplication by second order difference quotients of $\varrho$ at $x_0$, whose parameters are of order $\epsilon^{-2}$. For the ReLU this is neither possible nor necessary, since the corresponding network of \cite[Proposition C.2]{guehring20} improves its accuracy by depth alone. The difference is thus caused by the admitted activation functions and not by the localization by a partition of unity.
\end{remark}

\newgeometry{margin=0.8in}
\setstretch{0.9}
\bibliographystyle{plainurl}
{\footnotesize\bibliography{references}}

@ARTICLE{mhaskar,
    AUTHOR = {Mhaskar, H. N.},
    TITLE = {Neural Networks for Optimal Approximation of Smooth and Analytic Functions},
    JOURNAL = {Neural Computation},
    YEAR = {1996},
    volume={8},
    number={1},
    pages={164--177},
    doi={https://doi.org/10.1162/neco.1996.8.1.164}
}

@ARTICLE{petersen21,
	AUTHOR = {Petersen, Philipp and Raslan, Mones and Voigtlaender, Felix},
	TITLE = {Topological Properties of the set of functions generated by neural networks of fixed size},
	JOURNAL = {Foundations of Computational Mathematics},
	YEAR = {2021},
	volume={21},
	number={2},
	pages={375--444},
	doi={https://doi.org/10.1007/s10208-020-09461-0}
}

@ARTICLE{lu_main,
	AUTHOR = {Lu, Jianfeng and Shen, Zuowei and Yang, Haizhao and Zhang, Shijun},
	TITLE = {Deep Network Approximation for Smooth Functions},
	JOURNAL = {SIAM Journal on Mathematical Analysis},
	YEAR = {2021},
	volume={53},
	number={5},
	pages={5465--5506},
	doi={https://doi.org/10.1137/20M134695X}
}

@ARTICLE{shen_external1,
	AUTHOR = {Zuowei Shen and Haizhao Yang and Shijun Zhang},
	TITLE = {Nonlinear approximation via compositions},
	JOURNAL = {Neural Networks},
	YEAR = {2019},
	volume={119},
	pages={74--84},
	doi={https://doi.org/10.1016/j.neunet.2019.07.011}
}

@ARTICLE{shen_external2,
	AUTHOR = {Zuowei Shen and Haizhao Yang and Shijun Zhang},
	TITLE = {Deep Network Approximation Characterized by Number of Neurons},
	JOURNAL = {Communications in Computational Physics},
	YEAR = {2020},
	volume={28},
	number={5},
	pages={1768--1811},
	doi={https://doi.org/10.4208/cicp.OA-2020-0149}
}

@BOOK{Katznelson2004,
	AUTHOR = {Katznelson, Yitzhak},
	YEAR = {2004},
	TITLE = {An Introduction to Harmonic Analysis},
	PUBLISHER = {Cambridge University Press},
	ADDRESS = {Cambridge},
	doi={https://doi.org/10.1017/CBO9781139165372}
}

@article{guehring20,
AUTHOR = {G\"{u}hring, Ingo and Kutyniok, Gitta and Petersen, Philipp},
TITLE = {Error bounds for approximations with deep {ReLU} neural networks in ${W}^{s,p}$ norms},
JOURNAL = {Analysis and Applications},
YEAR = {2020},
volume={18},
number={5},
pages={803--859},
doi={https://doi.org/10.1142/S0219530519410021}
}

@article{Petersen18,
	title = {Optimal approximation of piecewise smooth functions using deep {ReLU} neural networks},
	journal = {Neural Networks},
	year = {2018},
	author = {Philipp Petersen and Felix Voigtlaender},
	volume={108},
	pages={296--330},
	doi={https://doi.org/10.1016/j.neunet.2018.08.019}
}

@article{devore21, 
	title={Neural network approximation}, journal={Acta Numerica},
	author={DeVore, Ronald and Hanin, Boris and Petrova, Guergana}, 
	year={2021},
	volume={30},
	pages={327--444},
	doi={https://doi.org/10.1017/S0962492921000052}
}

@article{gribonval20,
	author={Gribonval, Rémi and Kutyniok, Gitta and Nielsen, Morten and Voigtlaender, Felix},
	year={2022},
	title={Approximation Spaces of Deep Neural Networks},
	journal={Constructive Approximation},
	volume={55},
	number={1},
	pages={259--367},
	doi={https://doi.org/10.1007/s00365-021-09543-4}
}

@article{elbraecther21,
	author={Elbr\"{a}chter, Dennis and Perekrestenko, Dmytro and Grohs, Philipp and B\"{o}lcskei, Helmut},
	year={2021},
	title={Deep Neural Network Approximation Theory},
	journal={IEEE Transactions on Information Theory},
	volume={67},
	number={5},
	pages={2581--2623},
	doi={https://doi.org/10.1109/TIT.2021.3062161}
}

@article{deryck21,
  title={On the approximation of functions by tanh neural networks},
  author={De Ryck, Tim and Lanthaler, Samuel and Mishra, Siddhartha},
  journal={Neural Networks},
  volume={143},
  pages={732--750},
  year={2021},
  publisher={Elsevier},
  doi={https://doi.org/10.1016/j.neunet.2021.08.015}
}

@article{belomestny23,
title = {Simultaneous approximation of a smooth function and its derivatives by deep neural networks with piecewise-polynomial activations},
journal = {Neural Networks},
year = {2023},
author = {Denis Belomestny and Alexey Naumov and Nikita Puchkin and Sergey Samsonov},
volume={161},
pages={242--253},
doi={https://doi.org/10.1016/j.neunet.2023.01.035}
}

@article{langer21,
title = {Approximating smooth functions by deep neural networks with sigmoid activation function},
journal = {Journal of Multivariate Analysis},
year = {2021},
author = {Sophie Langer},
volume={182},
pages={104696},
doi={https://doi.org/10.1016/j.jmva.2020.104696}
}

@article{raslan21,
title = {Approximation rates for neural networks with encodable weights in smoothness spaces},
journal = {Neural Networks},
year = {2021},
author = {Ingo Gühring and Mones Raslan},
volume={134},
pages={107--130},
doi={https://doi.org/10.1016/j.neunet.2020.11.010}
}

@article{jentzen_welti,
title = {Overall error analysis for the training of deep neural networks via stochastic gradient descent with random initialisation},
journal = {Applied Mathematics and Computation},
year = {2023},
author = {Arnulf Jentzen and Timo Welti},
volume={455},
pages={127907},
doi={https://doi.org/10.1016/j.amc.2023.127907}
}

@article{jentzen_riekert,
	title = {Strong overall error analysis for the training of artificial neural networks via random initializations},
	journal = {Communications in Mathematics and Statistics},
	year = {2024},
	volume={12},
	number={3},
	pages={385--434},
	author = {Arnulf Jentzen and Adrian Riekert},
	doi={https://doi.org/10.1007/s40304-022-00292-9}
}

@InProceedings{yarotsky21elementary,
	title = 	 {Elementary superexpressive activations},
	author =       {Yarotsky, Dmitry},
	booktitle = 	 {Proceedings of the 38th International Conference on Machine Learning},
	pages = 	 {11932--11940},
	year = 	 {2021},
	editor = 	 {Meila, Marina and Zhang, Tong},
	volume = 	 {139},
	series = 	 {Proceedings of Machine Learning Research},
	month = 	 {18--24 Jul},
	publisher =    {PMLR},
	url = 	 {https://proceedings.mlr.press/v139/yarotsky21a.html},
}

@article{boelcskei19optimal,
	author = {B{\"o}lcskei, Helmut and Grohs, Philipp and Kutyniok, Gitta and Petersen, Philipp},
	title = {Optimal Approximation with Sparsely Connected Deep Neural Networks},
	journal = {SIAM Journal on Mathematics of Data Science},
	volume = {1},
	number = {1},
	pages = {8-45},
	year = {2019},
	doi = {10.1137/18M118709X}
}

@article{maiorov99,
	title = {Lower bounds for approximation by {MLP} neural networks},
	journal = {Neurocomputing},
	volume = {25},
	number = {1--3},
	pages = {81-91},
	year = {1999},
	issn = {0925-2312},
	doi = {https://doi.org/10.1016/S0925-2312(98)00111-8},
	author = {Vitaly Maiorov and Allan Pinkus}
}

@article{yarotsky17error,
title = {Error bounds for approximations with deep {ReLU} networks},
journal = {Neural Networks},
year = {2017},
author = {Dmitry Yarotsky},
volume={94},
pages={103--114},
doi={https://doi.org/10.1016/j.neunet.2017.07.002}
}

\end{document}